\documentclass[twoside]{article}

\usepackage[preprint]{aistats2027}

\usepackage[round]{natbib}

\usepackage[utf8]{inputenc}
\usepackage[T1]{fontenc}
\usepackage{microtype}

\usepackage{amsmath}
\usepackage{amssymb}
\usepackage{amsfonts}
\usepackage{mathtools}
\usepackage{amsthm}
\usepackage{bm}
\usepackage{nicefrac}

\usepackage[small]{caption}
\usepackage{subcaption}
\usepackage{graphicx}
\usepackage{float}
\usepackage{booktabs}
\usepackage{multirow}
\usepackage{enumitem}
\usepackage[table]{xcolor}
\usepackage{tikz}

\usepackage{url}
\usepackage{hyperref}
\definecolor{mydarkblue}{rgb}{0,0.08,0.45}
\hypersetup{
  colorlinks=true,
  linkcolor=mydarkblue,
  citecolor=mydarkblue,
  urlcolor=mydarkblue
}
\usepackage[capitalise,nameinlink]{cleveref}

\theoremstyle{plain}
\newtheorem{proposition}{Proposition}
\newtheorem{lemma}{Lemma}
\newtheorem{corollary}{Corollary}
\theoremstyle{definition}

\newtheorem{assumption}{Assumption}
\theoremstyle{remark}

\crefname{assumption}{Assumption}{Assumptions}

\DeclareMathOperator*{\argmin}{arg\,min}
\DeclareMathOperator*{\argmax}{arg\,max}
\DeclareMathOperator{\supp}{supp}
\DeclareMathOperator{\Cov}{Cov}
\newcommand{\KL}{D_{\mathrm{KL}}}
\newcommand{\Dn}{\mathcal{D}_n}
\newcommand{\loss}{\mathcal{L}}
\newcommand{\losse}{\mathcal{L}_{\mathrm{emp}}}
\newcommand{\TM}{\mathcal{T}_M}
\newcommand{\Ssp}{\mathcal{S}}
\newcommand{\E}{\mathbb{E}}
\newcommand{\rt}{\tilde r_\Pi}
\newcommand{\Pcov}{\mathcal{P}}
\newcommand{\bth}{\bm{\theta}}
\newcommand{\bx}{\mathbf{x}}
\newcommand{\bX}{\mathbf{X}}

\newcommand{\fakeparagraph}[1]{\textbf{#1}\quad}

\begin{document}

\runningtitle{How Well Do Prior-Data Fitted Networks Amortize Bayesian Inference?}

\twocolumn[

\aistatstitle{How Well Do Prior-Data Fitted Networks Amortize \\ Bayesian Inference?}

\aistatsauthor{ Manuel H\"ulskamp \quad Julius Kobialka  \quad Emanuel Sommer  \quad Thomas Nagler  \quad David R\"ugamer }

\aistatsaddress{LMU Munich and Munich Center for Machine Learning } ]

\begin{abstract}
Prior-data fitted networks (PFNs) are trained to approximate the Bayesian posterior predictive distribution (PPD) of a synthetic prior, yet results about the approximation quality are surprisingly thin. Taking an empirical risk minimization perspective, we decompose the deployment error of PFNs into prior, context, and amortization error, and measure the latter against a reference PPD under a realistic prior. Results show that PFNs generally recover the targeted PPD very well.  Nevertheless, we find that parts of the residual amortization error persist, especially in \textit{rare} settings under the prior pretraining distribution, even when scaling the model size, which manifests in disagreement between independently pretrained PFNs and leaves uncertainty about the predictor at inference time. As a remedy, we study multiple types of Bayesian model averages and find that these can consistently reduce this uncertainty and improve predictions up to full-size TabICLv2 models.
\end{abstract}

\section{Introduction}\label{sec:intro}
\begin{figure}[t]
    \centering
    \tikzset{
        ob/.style={draw, rounded corners=2pt, align=center, inner sep=2.5pt, line width=0.5pt, font=\scriptsize},
        otop/.style={ob, anchor=north, minimum height=1.15cm},
        oleaf/.style={ob, anchor=north west, line width=0.4pt},
        oprior/.style={draw=red!30!black!45, fill=red!3, text=black!65},
        oamort/.style={draw=blue!65!black, fill=blue!14, line width=1.1pt},
        octx/.style={draw=green!30!black!45, fill=green!3, text=black!65},
        lprior/.style={draw=red!30!black!40, fill=red!1, text=black!65},
        lamort/.style={draw=blue!65!black, fill=blue!6, line width=0.6pt},
        olink/.style={draw=black!35, line width=0.5pt},
        oalink/.style={draw=blue!65!black, line width=0.8pt},
    }
    \begin{tikzpicture}
    \node[ob, anchor=north, draw=black!70, fill=black!4, text width=5.2cm] (root) at (0,0)
        {\textbf{Deployment error} $R_{\Pi'}(q)-H_{\Pi'}$ {\color{black!55}(\S\ref{sec:decomp_method})}};

    \node[otop, oprior, text width=2.2cm] (prior) at (-2.75,-0.75) {\textbf{Prior error}};
    \node[otop, oamort, text width=2.7cm] (amort) at (0,-0.75)
        {\textbf{Amortization error}\\[1pt] $\mathcal{A}_\Pi(q)$};
    \node[otop, octx, text width=2.2cm] (ctx) at (2.75,-0.75) {\textbf{Context error}\\[1pt] $\mathcal{C}_\Pi$};

    \draw[olink] (root.south) -- ++(0,-0.1) -| (prior.north);
    \draw[olink] (root.south) -- ++(0,-0.1) -| (ctx.north);
    \draw[oalink] (root.south) -- (amort.north);

    \node[oleaf, lprior, text width=1.75cm] (cov) at ([xshift=0.3cm, yshift=-0.22cm]prior.south west)        {Coverage $\Pcov_{\Pi'}(\Pi)$};
    \node[oleaf, lprior, text width=1.75cm] (tra) at ([yshift=-0.12cm]cov.south west)
        {Transport $\lambda\,\Gamma_{\Pi'}(q)$};
    \node[oleaf, lprior, text width=1.75cm] (off) at ([yshift=-0.12cm]tra.south west)
        {Off-support\\ $(1{-}\lambda)\,\mathcal{O}_{\Pi'}(q)$};
    \foreach \t in {cov,tra,off}
        \draw[olink] ([xshift=0.15cm]prior.south west) |- (\t.west);

    \node[oleaf, lamort, text width=2.2cm] (opt) at ([xshift=0.3cm, yshift=-0.22cm]amort.south west)
        {Optimization};
    \node[oleaf, lamort, text width=2.2cm] (est) at ([yshift=-0.12cm]opt.south west)
        {Estimation};
    \node[oleaf, lamort, text width=2.2cm] (app) at ([yshift=-0.12cm]est.south west)
        {Approximation};
    \foreach \t in {opt,est,app}
        \draw[oalink] ([xshift=0.15cm]amort.south west) |- (\t.west);
    \end{tikzpicture}
    \caption{Decomposition of the deployment error of a PFN (\cref{prop:decomp}). This paper studies the amortization error, the only term controlled by pretraining, which splits into optimization, estimation, and approximation error (\cref{eq:decomp_eq}). The context and prior error support this analysis: the context error is incurred even by the exact PPD, and the prior error arises when the deployment tasks differ from the prior (see \cref{sec:decomp_method}).}
    \label{fig:overview}
\end{figure}
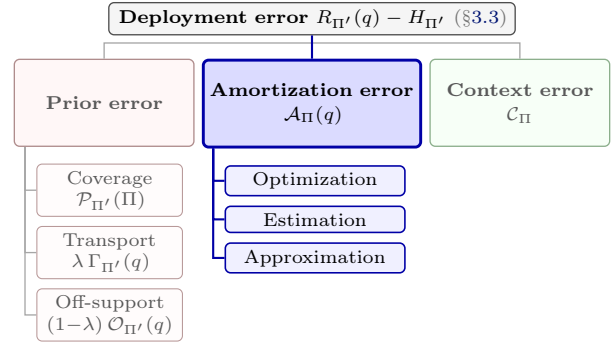
Prior-data fitted networks (PFNs) have emerged as powerful foundation models for tabular data \citep{TabPFN1_2021, tabPFN3_2025}. By learning to approximate Bayesian posterior predictive distributions (PPDs) via in-context learning \citep{ICL} during pretraining, they perform inference on new downstream datasets in a single forward pass. Given their success, many works have proposed strategies to improve predictive performance, scale to larger datasets, or transfer the approach to other tasks \citep{tabicl2, zhang2025mitra, breejenTabForest, thomasLocalPFN, xuMicturePFN, LiuTabPFNUnleashed, ma2025foundation, robertson2025pfn, balazadeh2025causalpfn}.
The Bayesian interpretation is central to this success. A PFN is trained on datasets drawn from a synthetic prior, and its training objective is minimized by the PPD of that prior. In other words, a PFN is trained to minimize its risk under its own prior, and the PPD is the optimal predictor there.

Surprisingly, the evidence that trained PFNs closely track the PPD they amortize is thin and mostly focused on whether \textit{transformers can do Bayesian inference} at all via the PPD \citep{TabPFN1_2021} and not how well they recover it.
Moreover, pretraining is typically reduced to first-order optimization of a single network, on the grounds that effectively unlimited synthetic data should remove any uncertainty about the outcome. Whether this holds is a question about amortization uncertainty. Existing work on uncertainty in PFNs \citep{NaglerMartingale, FortiniDecomposition} studies the uncertainty that remains for a \emph{fixed} network at inference time and is, by construction, blind to it.

Furthermore, when PFNs are deployed on real data, the model likely faces a different distribution of tasks, and the optimal predictor changes with it. Hence, deployment is a distribution shift from the pretraining prior to the real-world task distribution, which further complicates the error source analysis of deployed PFNs.

\textbf{Our Contribution}\quad In this work, we make the following contributions to study the sources of error and uncertainty in PFNs with a focus on amortization:
\begin{enumerate}[leftmargin=1.5em, nosep]
    \item A decomposition of the error of any PFN under a deployment distribution, splitting into prior, context and amortization error.
    \item A study of the amortization error of PFNs under its training prior, including a comparison to a realistic reference PPD (\cref{sec:decomp}).
    \item An extended investigation of the amortization error of PFNs under a deployment task distribution (\cref{sec:deployment}).
    \item An analysis of approximate Bayesian model averages over PFN weights (\cref{sec:bayes}): Deep ensembles consistently improve predictions, whereas sampling and variational learning add little.
\end{enumerate}

\section{Amortization error}\label{sec:decomp}
\subsection{Setup}\label{sec:setup}
Following \citet{NaglerPFNs}, a \emph{task} is a law $p$ over $(\bX,Y)\in\mathcal{X}\times\mathcal{Y}$ with covariate law $p_{\bX}$ and conditional $p(y\,|\,\bx)$, and a dataset $\Dn=\{(\bX_i,Y_i)\}_{i=1}^n$ consisting of $n$ independent and identically distributed ({iid}) draws from it. A PFN assumes a \emph{prior} $\Pi$ over pairs $(p,N)$ of a task and a context size, whose marginal over $N$ is the \emph{size distribution} $\Psi$. Given $(p,N{=}n)$, the context is $\Dn\sim p^{\otimes n}$, and the PFN targets the posterior predictive distribution (PPD)
\begin{equation}\label{eq:ppd}
  \pi_\Pi(y\mid\bx,\Dn)=\int p(y\mid\bx)\,d\Pi(p\mid\bx,\Dn).
\end{equation}
At deployment, pairs follow a different distribution $\Pi'$ with size marginal $\Psi'$.

In the following, we study PFNs through the lens of risk minimization. Risks average over a task, a context $\Dn\sim p^{\otimes n}$ and an independent query $(\bX,Y)\sim p$. We write $\E_p[\cdot]:=\E_{\Dn\sim p^{\otimes n}}\E_{(\bX,Y)\sim p}[\cdot]$ for a fixed pair $(p,n)$ and $\E_\Pi[\cdot]:=\E_{(p,N)\sim\Pi}\E_p[\cdot]$ when the pair is drawn as well, likewise for $\Pi'$. Random variables are capitalized.

\subsection{Risk Under a Task Distribution} \label{sec:risk}
A PFN with parameters $\bth\in\Theta$ is a measurable map $q_{\bth}\colon(\bx,\Dn)\mapsto q_{\bth}(\cdot\mid\bx,\Dn)$ into the distributions on $\mathcal{Y}$.
Deployed on $\Pi'$, its risk under the log loss is
\begin{equation}\label{eq:risk}
  R_{\Pi'}(q_{\bth}):=\E_{\Pi'}\big[-\log q_{\bth}(Y\mid\bX,\mathcal D_N)\big].
\end{equation}
The following proposition characterizes the optimal predictor under an arbitrary task distribution.

\begin{proposition}\label{prop:bayes}
Let $R_{\Pi'}(\pi_{\Pi'})$ be finite. Then for every predictor $q$,
\begin{align}\label{eq:bayes-gap}
  R_{\Pi'}&(q)-R_{\Pi'}(\pi_{\Pi'})= \notag \\
  &\E_{\Pi'}\big[\KL\big(\pi_{\Pi'}(\cdot\mid\bX,\mathcal D_N)\,\big\|\,q(\cdot\mid\bX,\mathcal D_N)\big)\big] \ge 0,
\end{align}
with equality if and only if $q=\pi_{\Pi'}$ almost surely under $\Pi'$. In particular, $\pi_{\Pi'}$ is the unique minimizer of $R_{\Pi'}$ and $R_{\Pi'}(\pi_{\Pi'})$, the \emph{Bayes risk} of $\Pi'$, is the lowest risk any predictor can attain on $\Pi'$.
\end{proposition}

The proof is given in \cref{app:proof-bayes}. A PFN $q_{\bth}$ is pretrained on the population objective
{\begin{equation}\label{eq:obj}
  \loss(\bth):=R_{\Pi}(q_{\bth}).
\end{equation}}
With $\Pi'=\Pi$, \cref{prop:bayes} thus gives the pretraining objective: \emph{a PFN is trained to minimize its risk under its own prior}, the minimizer over all predictors is $\pi_\Pi$ \citep{TabPFN1_2021}, and the excess risk
\begin{equation}\label{eq:amort}
  \mathcal{A}_\Pi(q_{\bth}):=R_\Pi(q_{\bth})-R_\Pi(\pi_\Pi)=\E_\Pi\big[\KL\big(\pi_\Pi\,\|\,q_{\bth}\big)\big]\ \ge0
\end{equation}
is the \emph{amortization error}. This error is the distance, in expected KL, of the network from the PPD it is supposed to amortize.

In practice, \cref{eq:obj} is approximated by empirical risk minimization over $M$ datasets $\mathcal{D}^{(j)}$ whose task and size $(p_j,n_j)$ are drawn from the prior,
\begin{align}\label{eq:erm}
  \losse(\bth) := -\tfrac{1}{M}\textstyle\sum_{j=1}^M\log q_{\bth}\big(y_j\mid\bx_j,\mathcal{D}^{(j)}\big),
\end{align}
optimized with first-order stochastic gradient methods such that, in general, we have no guarantee how close we get to the target $\pi_\Pi$.

\Cref{fig:overview} places the amortization error among the two other sources of error of a deployed PFN: the \emph{context error}, incurred even by $\pi_\Pi$ because a finite context leaves the task uncertain, and the \emph{prior error}, incurred because real tasks follow $\Pi'\neq\Pi$. We focus on the former and work towards its precise relation to the other two in \cref{prop:decomp}. As a first step, we propose how to analyse the amortization error for a tractable but realistic reference prior in the following section.

\begin{figure*}[ht]
    \centering
    \includegraphics[width=0.92\linewidth]{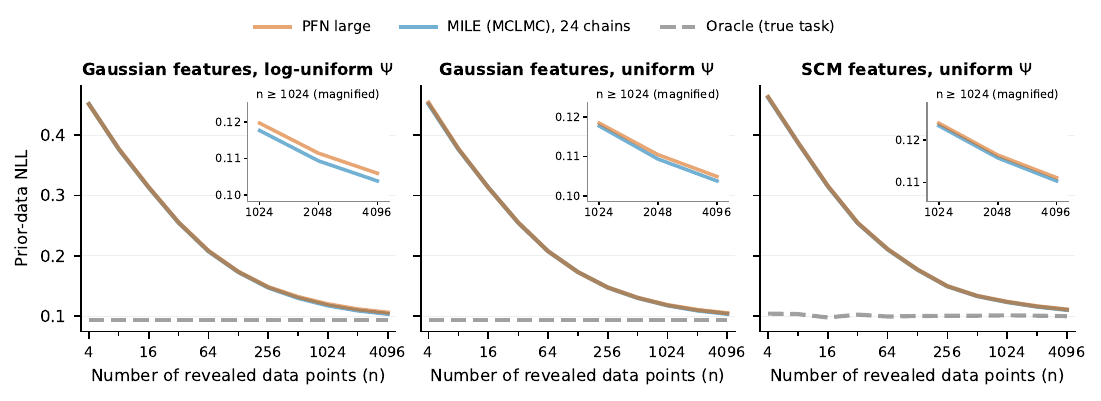}
    \caption{Comparison of PPD approximation quality to MCMC reference on a classification task across both Gaussian and TabICL prior feature distributions (SCM prior). The labeling mechanism is a Bayesian neural network with two hidden layers with 16 neurons each, and the large TabICL variant is trained on these priors from scratch. The PFN approximates the ground-truth PPD provided by the MCMC reference very well across different context sizes (Figure inspired by \citealp{TabPFN1_2021}). Each point is the risk on datasets of context size $n$; $\Psi$ is the size distribution of pretraining.}
    \label{fig:PFNMCMC}
\end{figure*}

\subsection{Measuring Amortization Error}\label{sec:measure_amort}
Measuring $\mathcal{A}_\Pi(q)=\E_\Pi[\KL(\pi_\Pi\|q)]$ directly requires knowledge of the true PPD, which is intractable for the priors used by current tabular foundation models. The evidence that trained PFNs closely track the PPD they amortize comes from \citet{TabPFN1_2021, muller2024bayes}. The considered settings are low-dimensional compared to current structural causal model (SCM) prior \citep{tabicl2, tabPFN3_2025},  none use a current PFN architecture at scale, and only small context sizes are considered. Furthermore, the MCMC reference performed worse than the PFN under consideration.

To fill this gap, we measure the amortization error not only in the setting introduced in \citet{TabPFN1_2021} with standard Gaussian features but also in a considerably more realistic setting by keeping the covariate generator of a state-of-the-art synthetic prior, so that the context tables have the marginal structure of current SCM priors (mixed feature types, correlations, non-Gaussian marginals). We replace the random mechanism mapping covariates to targets by a Bayesian neural network (BNN) $f_{\bm{w}}$ defining the likelihood $y\mid\bx,\bm{w}\sim p(y\mid f_{\bm{w}}(\bx))$ with $\bm{w}$ drawn from a Gaussian prior, analogous to the BNN prior introduced in \citet{hollmann2023tabpfn}. Because the covariate generator and $\bm{w}$ are drawn independently, this prior factorizes. Conditioning on $\bx^*$ leaves the posterior over $\bm{w}$ unchanged, and the PPD reduces to
\begin{equation}\label{eq:bnn_ppd}
  \pi_\Pi(y\mid\bx^*,\Dn)=\int p\big(y\mid f_{\bm{w}}(\bx^*)\big)\,d\Pi(\bm{w}\mid\Dn).
\end{equation}
Making use of recent advances in both MCMC algorithm development \citep{RobnikMCLMC} as well as sampling-based inference for BNNs \citep{EmanuelMILE, EmanuelSMILE}, we compare the prediction of PFNs against the PPD they are trained to amortize to obtain an empirical lower bound on the approximation error. MCMC yields a reference $\hat\pi_\Pi$ against which $\KL(\hat\pi_\Pi\|q)$ can be computed for each pair $(\bx^*,\Dn)$. In \cref{fig:PFNMCMC}, we plot the aggregate NLL of a PFN against a large-budget MCMC reference. In the inset plots, it is visible that the MCMC reference, where 24 MCMC chains are sampled per task, still dominates the PFN, however, the amortization gap appears small and mostly dependent on what datasets are probable under the prior (deviations at small context settings for $\Psi$ uniform vs. deviations at large context settings for $\Psi$ loguniform). Details are given in \cref{app:tracktablePrior}.

While this approach is sufficient to study the general capabilities of transformers to approximate the PPD and measure quantitatively how good the approximation is in terms of the amortization error, two aspects complicate the analysis once we move to production-grade synthetic priors and deployment on real-world datasets: We cannot compute the reference PPD, and deployment under a distribution $\Pi'$ different from the training distribution $\Pi$ introduces a distribution shift (see \cref{app:tracktablePrior}). The following studies these problems in more detail.

\section{Amortization Error Under Deployment}\label{sec:deployment}
\subsection{Amortization Error Without a Reference}\label{sec:quant_amort}
Beyond tractable priors, we cannot compute a reference PPD, and the amortization error has to be studied without an (approximate) ground-truth. Viewing the PFN with parameters  $\hat\bth$ via empirical risk minimization, the amortization error splits into three components
\begin{align}\label{eq:decomp_eq}
\mathcal{A}_\Pi(q_{\hat\bth})
  =\;&\underbrace{R_\Pi(q_{\hat\bth})-R_\Pi(q_{\tilde\bth_M})}_{\text{optimization}}
  +\underbrace{R_\Pi(q_{\tilde\bth_M})-R_\Pi(q_{\bth^*})}_{\text{estimation}} \notag\\
  &+\underbrace{R_\Pi(q_{\bth^*})-R_\Pi(\pi_\Pi)}_{\text{approximation}},
\end{align}
where $\tilde\bth_M$ minimizes the empirical risk for $M$ training samples and $\bth^*=\argmin_{\bth\in\Theta}\loss(\bth)$, i.e., minimizes the population risk (\ref{eq:obj}). The \textbf{estimation error} is the part that classical learning theory ties to the number of pretraining datasets $M$ and expects to decrease at a $1/M$ rate \citep{NaglerPFNs}. This is where PFNs profit from effectively unlimited synthetic data. The \textbf{approximation error} is a property of the hypothesis class and is reduced only by enlarging it, which extensive work on architectures attempts \citep{tabicl2, tabPFN4_2025,Kong2026}. The \textbf{optimization error} is a consequence of the optimizer and for overparameterized networks trained on a non-convex risk, it is not guaranteed to decrease with model size or training duration.

To confirm that the estimation error vanishes as expected, we can study $K$ models that share all optimizer settings but differ in seed and training data, at several scales $M_1<\dots<M_L$, and compute
\begin{equation}\label{eq:ee_metric}
  \widehat{EE}(M)=\tfrac{1}{K}\textstyle\sum_{k=1}^K\loss_{\mathrm{test}}(\bth_M^{(k)})-\loss_{\mathrm{test}}(\bth_{\mathrm{ref}}),
\end{equation}
where $\bth_{\mathrm{ref}}$ is a reference checkpoint with $M_{\mathrm{ref}}\gg M_L$. Averaging over the $K$ runs cancels the zero-mean optimization noise of individual runs.

\subsection{Shared \& Removable Amortization Error}\label{sec:shared_removable}

Optimization and approximation error cannot be isolated in the same way, since neither $\tilde\bth_M$ nor $\bth^*$ nor $\pi_\Pi$ is tractable. What is observable is the functional spread across $K$ independently trained ensemble members $q_1,\dots,q_K$ with identical $M$, model size, and optimizer settings. We measure this spread by the generalized Jensen--Shannon divergence $\widehat{AU}_{\mathrm{JSD}}$ between the members' predictive distributions and by the pairwise argmax disagreement rate $\widehat{AU}_{\mathrm{dis}}$, both averaged over test points (definitions in \cref{app:metrics}). Neither requires labels. The following classical identity \citep{cover2006elements} connects this spread to predictive performance.

\begin{lemma}\label{lem:ambiguity}
For any distribution $p$ on $\mathcal{Y}$ and predictors $q_1,\dots,q_K$ with mixture $\bar q=\frac1K\sum_{k=1}^K q_k$,
\begin{equation}\label{eq:ambiguity}
\tfrac1K\textstyle\sum_{k=1}^K\KL(p\,\|\,q_k)=\KL(p\,\|\,\bar q)+{\Delta_p},
\end{equation}
where ${\Delta_p}:=\E_{y\sim p}\big[\log\bar q(y)-\frac1K\sum_{k=1}^K\log q_k(y)\big]\ge0$, with equality if and only if all $q_k$ agree $p$-almost everywhere.
\end{lemma}
{$\Delta_p$ is the \emph{averaging gain}, i.e., the amount by which the mixture improves on its average member in expected log loss under $p$.} Applied with $p=\pi_\Pi$, \cref{lem:ambiguity} splits the amortization error itself:

\begin{corollary}\label{cor:bias_variance}
For PFNs $q_1,\dots,q_K$ with mixture $\bar q$,
\begin{equation}\label{eq:bias_variance}
  \tfrac1K\textstyle\sum_{k}\mathcal{A}_\Pi(q_k)=\underbrace{\mathcal{A}_\Pi(\bar q)}_{\text{shared}}+\underbrace{\E_\Pi\big[{\Delta_{\pi_\Pi}}\big]}_{\text{removable}\ \ge0}.
\end{equation}
\end{corollary}
Thus, the average amortization error of $K$ PFNs splits into two parts that differ in what can be observed without the PPD. The first, $\mathcal{A}_\Pi(\bar q)$, is the error the members have in common as they jointly deviate from $\pi_\Pi$, which cannot be measured without a reference as in \cref{sec:measure_amort}. The second term is the part of the amortization error that differs between the members and does not require a reference to compute. It equals the log-loss gain $\tfrac1K\sum_k R_\Pi(q_k)-R_\Pi(\bar q)$ of the mixture over its average member, which can be estimated on datasets drawn from the prior, and it vanishes if and only if the members agree. The disagreement metrics $\widehat{AU}_{\mathrm{JSD}}$ and $\widehat{AU}_{\mathrm{dis}}$ are therefore label-free probes of this visible part, and blind to the shared one.

Revisiting \cref{eq:decomp_eq}, if the estimation error vanishes at the expected rate while disagreement persists, the remaining amortization error is consistent with optimization error, possibly combined with residual approximation error: independent runs reach functionally distinct solutions that more data does not reconcile. To test whether the removable part is driven to zero by the usual scaling levers, we study $\widehat{EE}$ and both disagreement metrics along two axes, the number of pretraining datasets $M$ (data axis) and the model size (capacity axis). Results are depicted in \cref{fig:scaling_PFN_joint}.

Under $\Pi'\neq\Pi$, \cref{prop:bayes} says that the optimal predictor is $\pi_{\Pi'}$, not the $\pi_\Pi$ the PFN was trained towards. Deploying a PFN on real data is thus a \emph{distribution shift} from $\Pi$ to $\Pi'$. The next result makes its cost explicit and reconciles in what sense model averaging can be beneficial.

\begin{figure*}[t]
    \centering
    \includegraphics[width=1\linewidth]{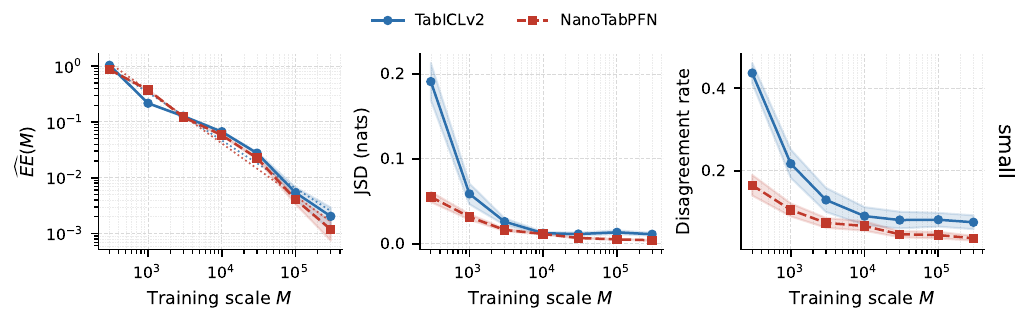}
    \caption{Estimation error $\widehat{EE}(M)$ (log-log) on synthetic data and member disagreement on TabArena ($\widehat{AU}_{\mathrm{JSD}}$, $\widehat{AU}_{\mathrm{dis}}$) versus the number of pretraining datasets $M$ for the small models; the tiny models are shown in \cref{fig:scaling_PFN_tiny}.}
    \label{fig:scaling_PFN_joint}
\end{figure*}

\subsection{Decomposing the Deployment Error}\label{sec:decomp_method}
If the PFN training on $\Pi$ does not perfectly align with the deployment distribution $\Pi'$, this difference needs to be taken into account. To locate the deployment pairs $(p,n)$ relative to the prior, the Lebesgue decomposition splits the deployment distribution uniquely into a part the prior can express $\tilde\Pi$, and a part $\Pi'_\perp$ it cannot,
\begin{equation}\label{eq:lebesgue}
  \Pi'=\lambda\,\tilde\Pi+(1-\lambda)\,\Pi'_\perp,\qquad \tilde\Pi\ll\Pi,\quad \Pi'_\perp\perp\Pi,
\end{equation}
with density $\rt=d\tilde\Pi/d\Pi$ on the absolutely continuous part. The weight $\lambda\in[0,1]$ is the fraction of deployment pairs reachable by reweighting the prior. The singular part $\Pi'_\perp$ lives on pairs to which $\Pi$ assigns probability zero, and no reweighting reaches it. These include tasks the prior cannot generate as well as context sizes outside the pretraining range. As a baseline for the deployment risk, let $H_{\Pi'}:=\E_{\Pi'}[-\log p(Y\mid\bX)]$ be the expected conditional entropy of the deployment tasks, i.e., the risk of an oracle that knows each task $p$ and predicts with $p(y\mid\bx)$. The excess risk of a PFN over this oracle then splits as follows.

\begin{proposition}[informal]\label{prop:decomp}
Under mild regularity conditions, for every PFN $q$,
\begin{align}
  R_{\Pi'}(q)-H_{\Pi'}
  =\;&\underbrace{\Pcov_{\Pi'}(\Pi)}_{\text{coverage}} \notag\\
  &+\lambda\Big[\underbrace{\mathcal{A}_\Pi(q)}_{\text{amortization}}+\underbrace{\mathcal{C}_\Pi}_{\text{context}}+\underbrace{\Gamma_{\Pi'}(q)}_{\text{transport}}\Big] \notag\\
  &+(1-\lambda)\,\underbrace{\mathcal{O}_{\Pi'}(q)}_{\text{off-support}}. \label{eq:decomp}
\end{align}
Coverage, amortization and context error are non-negative; the transport and off-support terms are sign-indefinite.
\end{proposition}
The formal statement with the definitions of all terms and its proof are given in \cref{app:proof-decomp}. The decomposition separates the deployment error of any PFN into three kinds of error with different levers (illustrated in \cref{fig:overview}), of which only one is controlled by pretraining.

\begin{table*}[t]
  \centering
  \caption{Relative LPPD gain (\%) of Bayesian model averages of TabICLv2 on 30 TabArena classification datasets at context sizes 1000 and 10000. \emph{DE}: deep ensemble of 8 members (4 for the large model) over its average member. \emph{IVON}: average over $S$ posterior samples over the posterior mean of the same model. \emph{SG-MCMC}: Bayesian deep ensemble over the DE it is warm-started from, sampled with pSMILE. $^{\ast}$: significant change, paired Wilcoxon over datasets, $p<0.05$.}
  \vspace{-0.25cm}
  \label{tab:ivon_vs_de}
  \small
  \setlength{\tabcolsep}{4.5pt}
  \begin{tabular}{lcccccccc}
    \toprule
    & \multicolumn{4}{c}{context size 1000} & \multicolumn{4}{c}{context size 10000} \\
    \cmidrule(lr){2-5}\cmidrule(lr){6-9}
    Model & DE & IVON $S{=}8$ & IVON $S{=}64$ & SG-MCMC & DE & IVON $S{=}8$ & IVON $S{=}64$ & SG-MCMC \\
    \midrule
    tiny   & $+3.10$ & $+2.38^{\ast}$ & $+2.48^{\ast}$ & {$+0.27^{\ast}$} & $+2.65$ & $-0.13$ & $-0.24$ & {$+0.35^{\ast}$} \\
    small  & $+2.89$ & $+0.52^{\ast}$ & $+0.65^{\ast}$ & {$+0.37^{\ast}$} & $+2.56$ & $-0.14$ & $-0.04$ & {$+0.42^{\ast}$} \\
    medium & $+2.67$ & $+0.42^{\ast}$ & $+0.44^{\ast}$ & {$+0.07$} & $+2.82$ & $-0.31$ & $-0.14$ & {$+0.08$} \\
    large  & $+1.69$ & $+0.23^{\ast}$ & $+0.18^{\ast}$ & {$+0.27^{\ast}$} & $+1.81$ & $-0.13$ & $-0.06$ & {$+0.32^{\ast}$} \\
    \bottomrule
  \end{tabular}
\end{table*}

\fakeparagraph{Context Error}
{$\mathcal{C}_\Pi=R_\Pi(\pi_\Pi)-H_\Pi$ is the excess risk of the exact PPD and does not involve the network: the price that even an exact Bayesian pays for a finite context. It is the object of \citet{NaglerMartingale} and \citet{ng2025tabmgp}, who separate it from aleatoric uncertainty via martingale posteriors \citep{Fong2023MPs}, conditional on a fixed network.}

\fakeparagraph{Prior Error}
{Coverage, transport and off-support terms describe how the prior relates to the deployment distribution. Coverage is model-free and vanishes if and only if every deployment task can be expressed by the prior. The transport term is the covariance under $\Pi$ between $\rt$ and the PFN's excess risk: it is positive when deployment overweights tasks on which the PFN is poor, so misplaced prior mass is not harmful \emph{per se}. The off-support term is the PFN's excess risk on $\Pi'_\perp$ relative to the closest expressible task and can be negative when the network generalizes beyond its prior \citep{RaventosPretrainingRegression,goddard2025can}. Quantifying these terms for PFNs remains open (\cref{app:related_work}).}

\fakeparagraph{Amortization Error}
{The amortization error enters \cref{eq:decomp} unchanged: it is still evaluated under the \emph{training} prior, weighted by the expressible share $\lambda$, and undefined on $\Pi'_\perp$, where there is no PPD to compare against. It remains the only term that pretraining controls. Importantly, under a deployment distribution $\Pi'\neq\Pi$, member disagreement no longer measures only the removable part $\E_\Pi[\Delta_{\pi_\Pi}]$ of the amortization error, since \cref{cor:bias_variance} identifies it only under the prior. On real data, \cref{lem:ambiguity} and the Lebesgue decomposition \cref{eq:lebesgue} combine into \cref{cor:deployment_amb} (\cref{app:proof-deployment}): the log-loss gain of an ensemble over its average member splits into the removable amortization error of \cref{cor:bias_variance} and a sign-indefinite transport term, both weighted by $\lambda$, plus a non-negative off-support term. Disagreement on real data, and with it the gain of averaging, therefore mixes uncertainty left by pretraining with how the deployment distribution relates to the prior. Since pretraining does not constrain the members on tasks the prior cannot express, we expect disagreement to grow as deployment moves away from the prior's support, e.g., for contexts beyond the pretraining range or datasets whose structure the prior does not generate.}

\section{Bayesian Model Averages Over PFN Weights}\label{sec:bayes}
If part of the amortization error persists and is partially removable in the sense of \cref{cor:bias_variance,cor:deployment_amb}, i.e., if the outcome of pretraining is itself uncertain and deployment adds additional uncertainty, a Bayesian model average of the network weights is the natural remedy.

\fakeparagraph{A Tempered Posterior Over PFN Weights}
Placing a Gaussian prior $p(\bth)=\mathcal{N}(0,\sigma^2I)$ on all weights and combining it with the PFN likelihood gives the posterior $p(\bth\mid\TM)\propto\exp(-M\losse(\bth))\,p(\bth)$, where $\TM$ denotes the collection of all pretraining datasets. Since $M$ can be made arbitrarily large and has no direct meaning as a training dataset size, this posterior collapses onto minimizers of \cref{eq:obj} as $M\to\infty$. To retain meaningful uncertainty, we instead target the \emph{tempered posterior} \citep{wenzelColdPosterior}
\begin{equation}\label{eq:tempered}
  p_\beta(\bth\mid\TM)\propto\exp\big(-\beta\,\losse(\bth)\big)\,p(\bth),
\end{equation}
in which $\beta>0$ plays the role of an effective sample size and decouples posterior concentration from $M$. Equivalently, $p_\beta$ is the unique minimizer over all distributions $Q$ on $\Theta$ of the Gibbs variational objective $\E_{Q}[\losse(\bth)]+\beta^{-1}\KL(Q\,\|\,p)$ \citep{BayesOpti}. In practice, its predictions are an average over $S$ samples from (an approximation of) $p_\beta$,
\begin{equation}\label{eq:bde_pred}
  \bar q(y\mid\bx,\Dn)=\tfrac{1}{S}\textstyle\sum_{s=1}^S q_{\bth^{(s)}}(y\mid\bx,\Dn).
\end{equation}
We compare three approximations of $p_\beta$ in \cref{sec:results_bde}.
\begin{figure*}[t]
    \centering
    \includegraphics[width=1\linewidth]{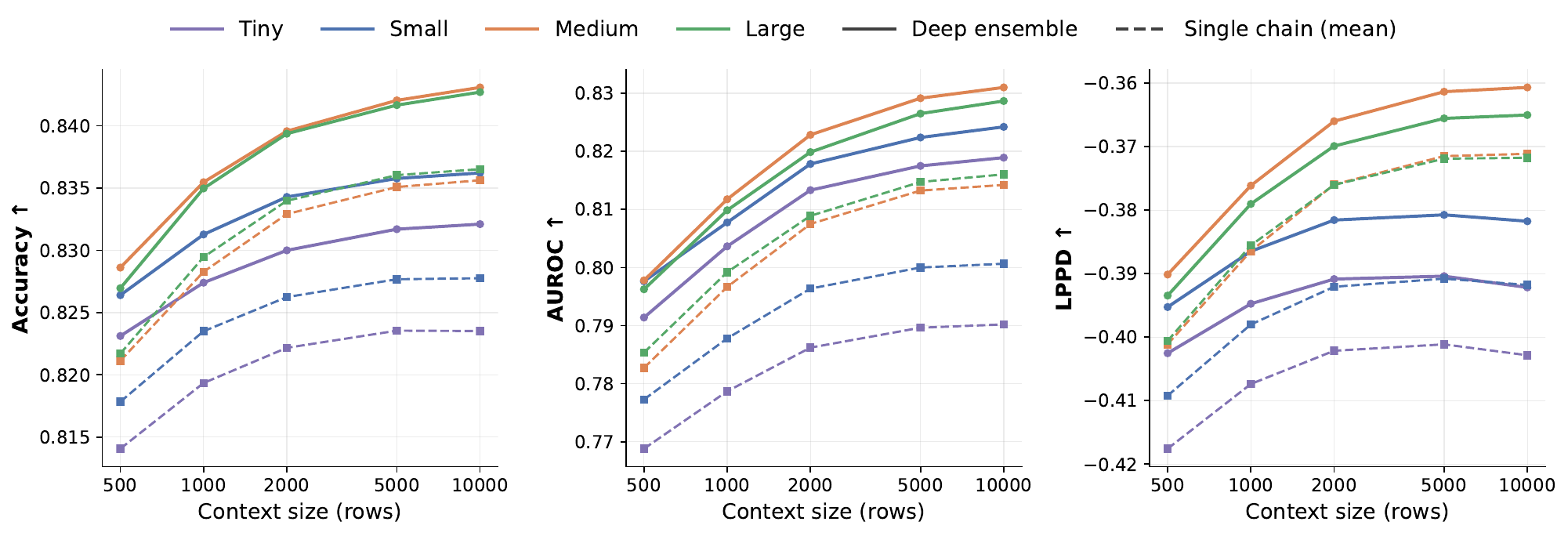}
    \caption{Classification test performance measured by accuracy, ROC AUC and LPPD on the TabArena dataset suite across a sweep of the maximal number of context rows for deep ensembles of TabICLv2 across a range of model sizes. The corresponding results for nanoTabPFN are shown in \cref{fig:DE-scaling-nanoTabPFN}.}
    \label{fig:DE-BDE-scaling}
\end{figure*}

\fakeparagraph{What Averaging Guarantees}
Applied with the true conditional $p$ of a deployment task, \cref{lem:ambiguity} guarantees that the expected log loss of $\bar q$ is never worse than that of the average member, and strictly better whenever the members disagree on that task, regardless of whether the task lies inside the support of the prior. On the expressible part it additionally reduces the amortization error by the removable part of \cref{cor:bias_variance}. Moreover, for finite $\mathcal{Y}$ the disagreement metrics are computable from the members alone, and $\widehat{AU}_{\mathrm{JSD}}>0$ on a dataset certifies ${\Delta_p}>0$ for every $p$ with full support, without labels. It hence can serve as a \textbf{label-free certificate of amortization uncertainty} (on $\Pi'_\perp$  $\widehat{AU}_{\mathrm{JSD}}>0$ only certifies a strict gain of $\bar q$ over its average member as no $\pi_\Pi$ exists).

\section{Experimental Study}\label{sec:experiments}
We now test the claims of \cref{sec:deployment,sec:bayes}: Whether the removable part of the amortization error persists under scaling (\cref{sec:results_scaling}), and whether approximate Bayesian methods can improve PFNs under real-world task distributions (\cref{sec:results_bde}).

\subsection{Setup}\label{subsec:setup}
We focus our experiments on TabICLv2 at up to four different scales: ($\sim$1.5M/2M/7.5M/9M). For some experiments we additionally consider  \textit{nanoTabPFN} \citep{nanotabPFN} at two sizes ($\sim$1M/2M parameters). \ Both are pretrained on synthetic tasks from the TabICL prior \citep{tabicl2} with up to 1024 rows, 100 features and 10 classes. Later, in \cref{sec:fullscale}, we extend the analysis to full-sized TabICLv2 models as well as regression.  We mostly study TabICLv2 because of its greatly improved training efficiency compared to nanoTabPFN and because of its much more refined architecture. Every ensemble member, and every sampling chain, is trained on disjoint pretraining data. Depending on the experiment, we evaluate accuracy (ACC), ROC AUC and log pointwise predictive density \citep[LPPD;][]{gelmanLPPD} on a large synthetic test set drawn from the TabICL prior or real-world classification datasets from TabArena \citep{TabArena} as well as TALENT \citep{liu2024talent}, potentially filtered to the scope of the models. Details on implementation, architectures, and optimization are given in \cref{app:setup}.

\subsection{Disagreement Along Scaling Axes}\label{sec:results_scaling}

\fakeparagraph{Data Axis}
At fixed model size, we compute $\widehat{EE}(M)$ in \cref{eq:ee_metric} and both disagreement metrics across $K{=}4$ members for the two smallest model sizes, keeping the number of training steps fixed. The two behave differently (\cref{fig:scaling_PFN_joint}): $\widehat{EE}(M)$ decreases steadily over the whole range for the small models of both architectures. A log-log fit over all $M$ gives $\widehat{EE}(M)\propto M^{-\alpha}$ with $\alpha\in[0.86,1.14]$ across architectures and sizes, consistent with the expected $1/M$ rate. Disagreement falls quickly for small $M$ but flattens beyond $M\approx10^4$. For larger training scale, the disagreement rate decreases only from $9.1\%$ to $7.6\%$ for TabICLv2 and from $6.7\%$ to $3.7\%$ for nanoTabPFN. $\widehat{AU}_{\mathrm{JSD}}$ also reduces by at most a factor of three. The tiny models and the synthetic test set show the same pattern (\cref{fig:scaling_PFN_tiny,fig:scaling_PFN_synthetic}). Within the scanned range, additional pretraining data therefore mainly reduces the estimation error, while a small but persistent share of the disagreement remains.

\fakeparagraph{Capacity Axis}
As shown in \cref{tab:capacity} in the Appendix, disagreement also decreases with capacity, but only slowly compared to the 30--50-fold drop of $\widehat{EE}$ along the data axis. Increasing the model parameters from 1.5M to 9M reduces $\widehat{AU}_{\mathrm{JSD}}$ on TabArena by $41\%$ and $\widehat{AU}_{\mathrm{dis}}$ from $7.5\%$ to $5.9\%$.
Contexts beyond the pretraining range do not reduce the disagreement rate either, and the full-size ensembles of \cref{sec:fullscale} still gain over their members. Capacity thus likely mainly lowers the error shared by all members, consistent with a smaller approximation error, whereas the removable part on prior data changes less.

\fakeparagraph{Reconciling Results}
\cref{sec:measure_amort} shows that a PFN closely tracks the PPD it is supposed to amortize. \cref{sec:results_scaling}, on the other hand, finds disagreement among $K$ independently trained PFNs on a (small) fraction of predictions. At first glance, this might seem contradictory. In both cases, however, the differences are caused by a small fraction of test rows: the difference to the MCMC reference concentrates on a small share of inputs (\cref{app:tracktablePrior}), and so does the averaging gain between independently trained members, where little gain falls on confidently predicted test points (\cref{fig:disagreement_concentration}). The test points with the largest averaging gains are plausibly those that are rare under the prior and hence contribute little to the pretraining objective, so that more pretraining or larger models barely reduces the uncertainty for these predictions. Averaging over independently pretrained models addresses exactly this residual. Under deployment, \cref{cor:deployment_amb} adds a transport and a non-negative off-support term, leading to a higher disagreement on datasets outside the prior's support as empirically confirmed in \cref{app:beyond}.

\begin{figure*}[t]
    \centering
    \includegraphics[width=1\linewidth]{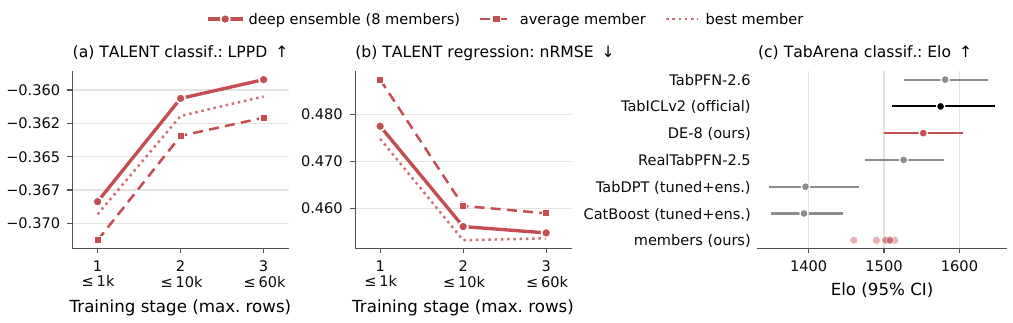}
    \caption{{Deep ensembles of full-size TabICLv2. \textbf{(a, b)} LPPD and nRMSE on TALENT datasets after each training stage: ensemble of 8 members, average member and best member (selected after the fact). The gap between ensemble and average member persists while all members improve. \textbf{(c)} TabArena classification Elo (95\% CI) of the single members (dots), the 8-member ensemble, the official TabICLv2 checkpoint and further published methods. Note that the individual DE members are trained for only 27.5\% of the steps that official TabICLv2 uses. }}
    \label{fig:FullTabICL}
\end{figure*}

\subsection{Bayesian Model Averages of PFNs}\label{sec:results_bde}
We approximate tempered posteriors of the form (\ref{eq:tempered}), with temperature and prior set per method, using three families of methods that differ in which part of the posterior they explore. All yield parameter samples that are averaged via \cref{eq:bde_pred}; implementation details and hyperparameters are given in \cref{app:bma_setup}.

\fakeparagraph{Deep Ensembles}
{Deep ensembles \citep[DEs;][]{Lakshminarayanan2017} train $K$ networks independently. These act as draws from distinct modes of $p_\beta$ and thus capture uncertainty \emph{between} basins \citep{FortLossLandscape}.} Deep ensembles improve on their average member in ACC, ROC AUC and LPPD at every model size and context size (\cref{fig:DE-BDE-scaling}), by $1.7$--$3.1\%$ LPPD within and $1.8$--$2.7\%$ beyond the pretraining range (\cref{tab:ivon_vs_de}). At 1000 rows, an ensemble of one size roughly matches a single member of the next larger size.

\fakeparagraph{Within-Basin Methods}
SG-MCMC and variational learning, instead, are more likely to capture uncertainty \emph{within} a basin. We run the samplers adaSGHMC \citep{ChenSGHMC,adaSGHMC} and pSMILE \citep{EmanuelSMILE}, starting chains from each DE member \citep{EmanuelConnecting}, which forms a Bayesian deep ensemble \citep{EmanuelMILE}. We further train IVON \citep{shen2024ivon}, which fits a diagonal Gaussian to $p_\beta$. Both within-basin methods recover part of the removable error, but less compared to the DE. IVON improves on its posterior mean within the pretraining range at every size, by $0.2$--$2.5\%$, shrinking with capacity and saturating by $S\approx8$ (\cref{app:IVON}). Beyond it, no number of samples helps (\cref{tab:ivon_scaling}). SG-MCMC chains (pSMILE) warm-started from the DE add up to $0.42\%$ on top of it at both context sizes. adaSGHMC fails to provide reliable additional gains despite an extensive hyperparameter search (\cref{app:SGMCMC}). Within-basin methods therefore can capture residual amortization uncertainty, but a plain DE delivers most of the attainable gain at matched inference cost, needs no tempering or sampler hyperparameters, and its members train independently.

\subsection{Deep Ensembles of Full-Size TabICLv2}\label{sec:fullscale}
Since DEs carry most of the gain, we test whether these findings persist at the scale of current state-of-the-art PFNs. We pretrain $K{=}8$ full-size TabICLv2 models ($27.5$M parameters) each for classification and regression, with the three-stage curriculum and preprocessing pipeline of \citet{tabicl2} but $27.5\%$ of its training steps (details and extended results in \cref{app:fullDE}). After every training stage, the ensemble improves on its average member, and it does so on all 38 classification and 13 regression datasets of TabArena. For classification, it is also better than the best single member. The same also holds true for evaluation on TALENT (see \cref{tab:full_de_stages}). Giving each member different preprocessing configurations leaves the ensemble unchanged, so the gain stems from the independently pretrained weights. With a TabArena classification Elo of $1552$, the ensemble lies within the confidence interval of the official TabICLv2 checkpoint ($1575$), while its members never communicate and can be trained in parallel in about $0.37\times$ the official wall-clock time.

\section{Discussion}\label{sec:discussion}
PFNs amortize the PPD of their prior remarkably well
up to small deviations, which are confined to inputs that are rare under the prior.
For production priors without a reference, only the removable part of the amortization error is observable, and it shows the same picture: While the estimation error vanishes at the expected rate, a small but consistent disagreement between independently pretrained PFNs remains, concentrated on few inputs and barely reduced by more pretraining data or capacity. Model averaging addresses exactly this remaining uncertainty, and deep ensembles of PFNs deliver most of the attainable gain, while within-basin methods like SG-MCMC and variational inference provide little additional benefit. At full scale, an ensemble of independently pretrained TabICLv2 models matches the official checkpoint at a fraction of its wall-clock time, with the deep ensemble outperforming the best post-hoc selected single model on TabArena as well as TALENT. For practitioners, this suggests a simple recipe: pretraining several members in parallel instead of a single model yields better predictions and makes the amortization error partially measurable.

Our analysis has limitations. The amortization gap is measured against finite-budget MCMC, not the exact PPD, and some studies use smaller models. Furthermore, while we expect our results to hold across the whole range of state-of-the-art PFN architectures, we empirically validate our claims for TabICLv2 and nanoTabPFN only. Next steps are studying adjustments to the pretraining pipelines to improve the amortization at the tails of the prior as well as investigating how to integrate the disagreement metrics as uncertainty signals at deployment.

\subsection*{AI Use Statement}
In this work, we used generative AI tools to support the implementation of the presented methods in code and to draft and review mathematical proofs.
We have not used generative AI tools to formulate research hypotheses and develop the theoretical and conceptual framework. Furthermore, other required disclosures such as data cleaning or data analysis are not applicable to this work.
Additionally, we used generative AI tools to assist in the writing and presentation of this manuscript by improving readability and the presentation of the included figures. Furthermore, generative AI tools were used to explore and identify relevant literature in addition to our initial manual survey. We have reviewed all AI-assisted work. AI-generated code was reviewed manually as well as tested in detail. All modifications of the manuscript suggested by AI tools were carefully inspected. We take responsibility for the final content of this work, including text, claims or artifacts produced with the aid of generative AI.

\subsection*{Acknowledgements}
This work was supported by the Bavarian AI Foundation Model Initiative (AI-Bay), funded by the Bavarian State Ministry of Science and the Arts. The authors gratefully acknowledge the computational resources and support provided by the Erlangen National High Performance Computing Center (NHR@FAU) at Friedrich-Alexander-Universität Erlangen-Nürnberg (FAU). Furthermore, the authors gratefully acknowledge the scientific support and resources of the AI service infrastructure LRZ AI Systems provided by the Leibniz Supercomputing Centre (LRZ) of the Bavarian Academy of Sciences and Humanities (BAdW), funded by the Bayerisches Staatsministerium für Wissenschaft und Kunst (StMWK), primarily through the BayernKI initiative.

\bibliography{references}

\clearpage

\clearpage
\appendix
\crefalias{section}{appendix}
\crefalias{subsection}{appendix}
\crefalias{subsubsection}{appendix}
\thispagestyle{empty}
\onecolumn
\aistatstitle{How Well Do Prior-Data Fitted Networks Amortize \\ Bayesian Inference? \\ Supplementary Materials}

\section{Related Work}\label{app:related_work}
\paragraph{Prior-Data Fitted Networks} \citet{TabPFN1_2021} introduced PFNs as transformers trained on synthetic data to approximate Bayesian posterior predictive distributions via in-context learning, with \citet{NaglerPFNs} providing statistical foundations by showing that PFNs admit a frequentist ERM interpretation. This line of work has since catalysed a rapidly growing family of tabular foundation models \citep{tabPFN3_2025,tabPFN4_2025,realtabPFN,Kong2026,pfeiffer2026tabh2o,jager2026tabpfn,team2026xiaomi,eo2026exaone,Ma2024,zhang2025limix,zhang2025mitra}. Importantly, TabICL \citep{tabicl1, tabicl2} provides a fully open-source alternative with strong scaling to large datasets. Taken together, these works establish in-context learning over synthetic priors as the dominant paradigm for tabular foundation models.

\paragraph{Uncertainty in PFNs} Recent works address uncertainty in PFN predictions at inference time. \Citet{NaglerMartingale} construct Bayesian posteriors via martingale posterior sampling; \citet{FortiniDecomposition,ng2025tabmgp} prove an aleatoric/epistemic decomposition via a predictive CLT. \Citet{FortiniDecomposition} additionally argue that weight-space approaches to PFN uncertainty constitute a ``category error,'' since they concern the transformer's meta-parameters rather than the task at hand. We believe this characterization conflates two orthogonal levels of uncertainty: Context uncertainty is conditional on a fixed predictor, while the amortization error, and in particular its removable part (\cref{cor:bias_variance}), concerns the predictor itself, which stochastic pretraining makes random. Context-level methods cannot register cross-warmstart disagreement as uncertainty, yet our results will show this signal is measurable, persistent, and task-relevant. Closely related to the PFN literature is the stream of work focusing on \textit{simulation-based inference} \citep[SBI;][]{cranmer2020frontier}, which has also focused on the uncertainty in its approaches. Notably, the simulation gap problem, that is, prior misspecification degrading amortized approximations \citep{schmitt2023detecting},  corresponds to the prior error in our decomposition and is discussed in \cref{sec:decomp_method}.

\paragraph{Bayesian Deep Learning} Bayesian deep learning maintains distributions over network weights for uncertainty quantification \citep{OpinionBDL}. We build on SG-MCMC methods such as SGHMC \citep{ChenSGHMC, adaSGHMC} and the recent MILE/SMILE samplers \citep{EmanuelMILE, EmanuelSMILE}, as well as deep ensembles \citep{Lakshminarayanan2017}, which \citet{FortLossLandscape} attribute to exploration of distinct loss basins and \citet{DAngeloRepulsive} connect to a Wasserstein gradient-flow posterior. Variational learning with IVON \citep{shen2024ivon} scales natural-gradient variational inference to large networks. All three are instantiated here over PFN weights, targeting a tempered posterior \citep{wenzelColdPosterior} to prevent collapse at large meta-training scale $M$.

\paragraph{Are PFNs and In-Context Learners Bayesian?}
Whether trained in-context learners implement the Bayes rule of their pretraining distribution has been studied mostly outside tabular data. \citet{RaventosPretrainingRegression,goddard2025can} show for linear regression that transformers pretrained on few tasks match the Bayes rule of that finite task set, and above a task-diversity threshold instead match the Bayes rule of the wider task family, i.e., they depart from the pretraining prior. \citet{Falck2024IsPerspective} test the martingale property that every Bayesian predictive satisfies and find violations in large language models. For PFNs, \citet{TabPFN1_2021} compare against exact or sampling-based references for Gaussian-process and small Bayesian-neural-network priors. The concrete evidence that PFNs track the posterior predictive they are supposed to amortize comes mainly from \citet{TabPFN1_2021,muller2024bayes}. Unfortunately, the released code of \citet{TabPFN1_2021} covers only part of their experiments and, to our understanding, contains inconsistencies that narrow what some of the experiments show, which is one of the reasons why we revisit and extend their experiments in this work.

\paragraph{Evaluating Synthetic Priors}
Several recent works try to score synthetic priors without pretraining, by relating distances between synthetic and real datasets in (summary) dataset space to downstream performance \citep{Kaya2026TowardsModels, DaviesMindModels, ZhaoFromModelsb}, with limited success. In the terms of \cref{prop:decomp}, they target the prior error, where the decomposition makes explicit that it consists of a coverage term, a transport term, and an off-support term, which respond to different interventions.

\section{Proofs}\label{app:proofs}
Throughout, $\Pi$ and $\Pi'$ are laws over pairs $(p,n)$ of a task and a context size, and $\Dn$ denotes a context of realized size $n$. Quantities indexed by a task, such as $\mathcal{E}_p$, $\mathcal{E}^\star_p$ and $a_p$ below, are evaluated at the pair drawn with $p$.
We collect the standing assumptions.
\begin{figure}[h!]
    \centering
    \tikzset{
        box/.style={draw, rounded corners=2pt, align=left, anchor=west, inner sep=4pt,
                    line width=0.6pt, minimum height=10mm},
        root/.style={box, draw=black!70, fill=black!4, text width=3.8cm},
        lvl/.style={box, text width=3.6cm},
        leaf/.style={box, line width=0.45pt, text width=5.0cm},
        prior/.style={draw=red!55!black, fill=red!7},
        amort/.style={draw=blue!60!black, fill=blue!7},
        ctx/.style={draw=green!40!black, fill=green!7},
        neutral/.style={draw=black!70, fill=black!4},
        lprior/.style={draw=red!55!black, fill=red!2},
        lamort/.style={draw=blue!60!black, fill=blue!2},
        lshared/.style={draw=blue!60!black, fill=black!3, dashed},
        link/.style={draw=black!50, line width=0.55pt, rounded corners=3pt},
        weight/.style={font=\footnotesize, text=black!60, above, inner sep=2pt},
        panel/.style={font=\small\bfseries, anchor=north west},
        mtag/.style={font=\scriptsize, text=black!55},
    }
    \begin{tikzpicture}[font=\small]
    \node[panel] at (-0.1,4.75) {(a) Error of a single PFN};
    \node[root] (root) at (0,-0.9)
        {\textbf{Deployment error}\\[2pt] $R_{\Pi'}(q_{\bth})-H_{\Pi'}$};

    \node[lvl, prior] (prior) at (5.35,2.65) {\textbf{Prior error}};
    \node[lvl, amort] (amort) at (5.35,-1.4)
        {\textbf{Amortization error}\\[2pt] $\mathcal{A}_\Pi(q_{\bth})$\\[1pt]{\scriptsize\color{black!55}vs.\ reference PPD (Sec.~\ref{sec:measure_amort})}};
    \node[lvl, ctx]   (ctx)   at (5.35,-4.25)
        {\textbf{Context error}\\[2pt] $\mathcal{C}_\Pi$};

    \node[leaf, lprior] (cov) at (10.45,3.85)
        {\textbf{Coverage}\\[2pt] $\Pcov_{\Pi'}(\Pi)$};
    \node[leaf, lprior] (tra) at (10.45,2.65)
        {\textbf{Transport}\\[2pt] $\lambda\,\Gamma_{\Pi'}(q_{\bth})$};
    \node[leaf, lprior] (off) at (10.45,1.45)
        {\textbf{Off-support}\\[2pt] $(1-\lambda)\,\mathcal{O}_{\Pi'}(q_{\bth})$};

    \node[leaf, lamort] (opt) at (10.45,0.0)
        {\textbf{Optimization}\\[2pt] $R_\Pi(q_{\hat\bth})-R_\Pi(q_{\tilde\bth_M})$};
    \node[leaf, lamort] (est) at (10.45,-1.4)
        {\textbf{Estimation}\\[2pt] $R_\Pi(q_{\tilde\bth_M})-R_\Pi(q_{\bth^*})$\\[1pt]{\scriptsize\color{black!55}$\widehat{EE}(M)$ (Sec.~\ref{sec:results_scaling})}};
    \node[leaf, lamort] (app) at (10.45,-2.8)
        {\textbf{Approximation}\\[2pt] $R_\Pi(q_{\bth^*})-R_\Pi(\pi_\Pi)$};

    \path[link] (root.east) -- ++(0.5,0) |- (prior.west);
    \path[link] (root.east) -- ++(0.5,0) |- node[weight, pos=0.8] {$\times\,\lambda$} (amort.west);
    \path[link] (root.east) -- ++(0.5,0) |- node[weight, pos=0.8] {$\times\,\lambda$} (ctx.west);
    \foreach \t in {cov,tra,off}
        \path[link] (prior.east) -- ++(0.5,0) |- (\t.west);
    \foreach \t in {opt,est,app}
        \path[link] (amort.east) -- ++(0.5,0) |- (\t.west);
    \end{tikzpicture}

    \vspace{10pt}

    \begin{tikzpicture}[font=\small]
    \node[panel] at (-0.1,3.6) {(b) Averaging $K$ PFNs $q_1,\dots,q_K$ into their mixture $\bar q$};
    \node[root] (avg) at (0,0.15)
        {\textbf{Average member error}\\[2pt] $\tfrac1K\textstyle\sum_k R_{\Pi'}(q_k)-H_{\Pi'}$};

    \node[lvl, neutral] (ens)  at (5.35,2.3)
        {\textbf{Ensemble error}\\[2pt] $R_{\Pi'}(\bar q)-H_{\Pi'}$};
    \node[lvl, neutral] (gain) at (5.35,-0.6)
        {\textbf{Ensemble gain}\\[2pt] $\E_{\Pi'}[a_p]$\\[1pt]{\scriptsize\color{black!55}LPPD gain; {label-free: $\widehat{AU}_{\mathrm{JSD}}$}}};

    \node[leaf, lshared] (shr) at (10.45,2.3)
        {\textbf{Shared amortization}\\[2pt] $\mathcal{A}_\Pi(\bar q)$, decomposed as in (a)\\[1pt]{\scriptsize\color{black!55}not removed by averaging}};
    \node[leaf, lamort] (rem) at (10.45,0.6)
        {\textbf{Removable amortization}\\[2pt] $\lambda\,\E_\Pi\big[{\Delta_{\pi_\Pi}}\big]\ \ge0$};
    \node[leaf, lprior] (tra2) at (10.45,-0.6)
        {\textbf{Transport}\\[2pt] $\lambda\,\Cov_\Pi(\rt,a_p)$};
    \node[leaf, lprior] (off2) at (10.45,-1.8)
        {\textbf{Off-support}\\[2pt] $(1-\lambda)\,\E_{\Pi'_\perp}[a_p]\ \ge0$};

    \path[link] (avg.east) -- ++(0.5,0) |- (ens.west);
    \path[link] (avg.east) -- ++(0.5,0) |- (gain.west);
    \path[link] (ens.east) -- (shr.west);
    \foreach \t in {rem,tra2,off2}
        \path[link] (gain.east) -- ++(0.5,0) |- (\t.west);
    \end{tikzpicture}
    \caption{Error decompositions used in this paper. (a) The deployment error of a PFN splits into prior, amortization and context error (\cref{prop:decomp}); the latter two enter with the expressible share $\lambda$, and the prior error comprises the coverage, transport and off-support terms. For the trained network $q_{\hat\bth}$, the amortization error splits further into optimization, estimation and approximation error (\cref{eq:decomp_eq}). (b) For $K$ PFNs, the error of the average member splits into the error of their mixture $\bar q$, whose amortization error is the part shared by all members (\cref{cor:bias_variance}), and the gain of averaging, which splits into the removable amortization error and the transport and off-support terms of \cref{cor:deployment_amb}.}
    \label{fig:decomposition_tree}
\end{figure}

\begin{assumption}[Finiteness]\label{as:finite} All cross-entropies appearing below are finite.\end{assumption}
\begin{assumption}[Projection]\label{as:proj} For every task $p$, the I-projection $p^\star\in\argmin_{\tilde p\in\Ssp}\E_p[\KL(p(\cdot\mid\bX)\,\|\,\tilde p(\cdot\mid\bX))]$ exists and depends measurably on $p$, with $p^\star=p$ for $p\in\Ssp$.\end{assumption}
\begin{assumption}[Support]\label{as:supp} $\Pi(\Ssp\times\mathbb{N})=1$ for $\Ssp$ the support of the task marginal of $\Pi$.\end{assumption}
{For a prior $\Pi'$ over pairs, the joint law of $(p,N,\Dn,\bX,Y)$ is $d\Pi'(p,n)\,p^{\otimes n}(d\Dn)\,p_{\bX}(d\bX)\,p(Y\mid\bX)\,dY$;} we write $\bar m_{\Pi'}$ for the induced law of the observables $(\bX,\Dn)$. We use three elementary facts.
\begin{enumerate}[label=(F\arabic*), leftmargin=2.5em]
  \item\label{f:label} \emph{Label laws.} Given $(p,\bX,\Dn)$, the label has density $p(\cdot\mid\bX)$. Given $(\bX,\Dn)$, it has density $\pi_{\Pi'}(\cdot\mid\bX,\Dn)$: by the tower property over $p$ and Tonelli's theorem, $\mathbb{P}(Y\in C\mid\bX,\Dn)=\int\!\int_C p(y\mid\bX)\,dy\,d\Pi'(p\mid\bX,\Dn)=\int_C\pi_{\Pi'}(y\mid\bX,\Dn)\,dy$ for measurable $C\subseteq\mathcal{Y}$.
  \item\label{f:kl} \emph{Log-loss gaps are expected divergences.} If $Y$ has density $u(\cdot\mid Z)$ given $Z$ and $\log u(Y\mid Z)$ is integrable, then for every predictive density $w$, conditioning on $Z$ gives
  \[
  \E\big[\log u(Y\mid Z)-\log w(Y\mid Z)\big]=\E\big[\KL\big(u(\cdot\mid Z)\,\|\,w(\cdot\mid Z)\big)\big]\in[0,\infty],
  \]
  which vanishes if and only if $w(\cdot\mid Z)=u(\cdot\mid Z)$ almost surely.
  \item\label{f:com} \emph{Change of measure.} If $\nu\ll\mu$ are probability measures with density $r$, and $f$ and $rf$ are $\mu$-integrable, then $\E_\nu[f]=\E_\mu[rf]=\E_\mu[f]+\Cov_\mu(r,f)$, since $\E_\mu[r]=1$.
\end{enumerate}

\subsection{Proof of \cref{prop:bayes}}\label{app:proof-bayes}
\begin{proof}
By \ref{f:label}, $Y$ has density $\pi_{\Pi'}(\cdot\mid\bX,\mathcal D_N)$ given $(\bX,\mathcal D_N)$, and $\log\pi_{\Pi'}(Y\mid\bX,\mathcal D_N)$ is integrable because $R_{\Pi'}(\pi_{\Pi'})$ is finite. Applying \ref{f:kl} with $Z=(\bX,\Dn)$, $u=\pi_{\Pi'}$ and $w=q$ gives
\[
R_{\Pi'}(q)-R_{\Pi'}(\pi_{\Pi'})=\E_{\Pi'}\big[\log\pi_{\Pi'}(Y\mid\bX,\mathcal D_N)-\log q(Y\mid\bX,\mathcal D_N)\big]=\E_{\Pi'}\big[\KL(\pi_{\Pi'}\,\|\,q)\big]\ge0,
\]
with equality if and only if $q=\pi_{\Pi'}$ for $\bar m_{\Pi'}$-almost every $(\bX,\Dn)$. Hence $R_{\Pi'}(q)\ge R_{\Pi'}(\pi_{\Pi'})$ for every predictor $q$, and any $q$ attaining this bound equals $\pi_{\Pi'}$ $\bar m_{\Pi'}$-almost surely, i.e., $\pi_{\Pi'}$ is the minimizer of $R_{\Pi'}$, unique up to $\bar m_{\Pi'}$-null sets.
\end{proof}

\subsection{Formal Statement and Proof of \cref{prop:decomp}}\label{app:proof-decomp}
For a ground-truth task $p$, let $H_p:=\E_p[-\log p(Y\mid\bX)]$ be its conditional entropy and
\begin{equation}\label{eq:excess}
  \mathcal{E}_p(q_{\bth}):=\E_p\big[\KL\big(p(\cdot\mid\bX)\,\big\|\,q_{\bth}(\cdot\mid\bX,\Dn)\big)\big]
\end{equation}
its excess risk, so that $R_{\Pi'}(q_{\bth})-H_{\Pi'}=\E_{\Pi'}[\mathcal{E}_p(q_{\bth})]$ with $H_{\Pi'}:=\E_{\Pi'}[H_p]$ by \ref{f:label} and \ref{f:kl}. {Let $\Ssp$ be the support of the task marginal of $\Pi$ (\cref{as:supp})} and let $p^\star$ be the \emph{I-projection} of $p$ onto $\Ssp$ (\cref{as:proj}), the task in the prior's support closest to $p$ in expected conditional KL. It defines the \emph{coverage} error and the signed excess risk relative to $p^\star$,
\begin{align}
  \Pcov_{\Pi'}(\Pi)&:=\E_{\Pi'}\big[\KL\big(p(\cdot\mid\bX)\,\|\,p^\star(\cdot\mid\bX)\big)\big] \ge0, \label{eq:coverage}\\
  \mathcal{E}^\star_p(q_{\bth})&:=\E_p\big[\log p^\star(Y\mid\bX)-\log q_{\bth}(Y\mid\bX,\Dn)\big]. \label{eq:excess-star}
\end{align}

\begin{proposition}[Deployment decomposition, formal version of \cref{prop:decomp}]\label{prop:decomp_formal}
Under \cref{as:finite,as:proj,as:supp}, and if $\rt\,\mathcal{E}_p(q)$ is $\Pi$-integrable, \cref{eq:decomp} holds for every PFN $q$, where $\mathcal{A}_\Pi(q)$ is the amortization error \cref{eq:amort}, $\Pcov_{\Pi'}(\Pi)$ the coverage error \cref{eq:coverage},
\begin{equation}\label{eq:transport-offsupport}
  \Gamma_{\Pi'}(q):=\Cov_\Pi\big(\rt,\mathcal{E}_p(q)\big),\qquad
  \mathcal{O}_{\Pi'}(q):=\E_{\Pi'_\perp}\big[\mathcal{E}^\star_p(q)\big]
\end{equation}
the transport and off-support terms, and
\begin{equation}\label{eq:controlled-split}
  \mathcal{C}_\Pi := R_\Pi(\pi_\Pi)-H_\Pi=\E_\Pi\big[\KL\big(p(\cdot\mid\bX)\,\|\,\pi_\Pi(\cdot\mid\bX,\mathcal D_N)\big)\big]\ \ge0
\end{equation}
the context error.
\end{proposition}
Before moving to the proof, the following paragraph highlights the intuition and relevance of each of the five terms in \cref{prop:decomp}:
\begin{itemize}[leftmargin=1.2em, nosep]
  \item \textbf{Coverage} $\Pcov_{\Pi'}(\Pi)$ depends on $\Pi$ and $\Pi'$ only: it is the misspecification of the prior, measured by the KL divergence to the best expressible task. Importantly, it only compares the conditionals of $Y$ given $X$, not the covariate law alone. Hence, a task whose covariate law the prior never generates, but whose conditional it does, incurs no coverage error.
  \item \textbf{Amortization} $\mathcal{A}_\Pi(q)$ is the only term the pretraining objective targets: it vanishes if and only if $q=\pi_\Pi$ (\cref{prop:bayes}). $\Gamma_{\Pi'}$ and $\mathcal{O}_{\Pi'}$ also depend on $q$, but they are not controlled by the objective as they also depend on the deployment distribution, which is unknown at training time. It is evaluated under $\Pi$, not $\tilde\Pi$, the reweighting towards deployment is accounted for in $\Gamma_{\Pi'}$.
  \item \textbf{Context} $\mathcal{C}_\Pi$: by \ref{f:label}, $R_\Pi(\pi_\Pi)$ and $H_\Pi$ are the conditional entropies of $Y$ given $(\bX,\mathcal D_N)$ and given $(p,\bX,\mathcal D_N)$, so $\mathcal{C}_\Pi=I(Y;p\mid\bX,\mathcal D_N)$, the information the unknown task still carries about the label. It is a property of $\Pi$ alone and does not depend on the PFN $q$. Its decay in the context size $n$ measures how quickly the prior's tasks are identified from data, i.e., how hard the prior is independently of the architecture.
  \item \textbf{Transport} {$\Gamma_{\Pi'}(q)$ measures whether deployment favours tasks or context sizes on which $q$ does well or poorly. If $q$ is equally good on all pairs,} it vanishes, however strongly $\Pi'$ reweights the prior. It is generally non-zero even for the exact PPD $q=\pi_\Pi$, since some tasks are harder to identify from $n$ points than others.
  \item \textbf{Off-support} $\mathcal{O}_{\Pi'}(q)$ {compares $q$ with the best expressible task $p^\star$ on pairs to which the prior assigns no mass, i.e., inexpressible tasks or context sizes outside the pretraining range.} Where their observables $(\bX,\Dn)$ fall outside the support of $\bar m_\Pi$, $\pi_\Pi$ is not determined and the pretraining objective does not constrain $q$, so the term is governed by the inductive bias of the architecture. $\mathcal{O}_{\Pi'}(q)<0$ means the network generalizes beyond its prior, as observed for in-context learners pretrained on sufficiently diverse tasks \citep{RaventosPretrainingRegression,goddard2025can}.
\end{itemize}

\begin{proof}
Let $c_p:=\E_p[\KL(p(\cdot\mid\bX)\,\|\,p^\star(\cdot\mid\bX))]$, so $\Pcov_{\Pi'}(\Pi)=\E_{\Pi'}[c_p]$. Splitting $\log p-\log q=(\log p-\log p^\star)+(\log p^\star-\log q)$ in \cref{eq:excess} gives, for every task $p$,
\begin{equation}\label{eq:telescope}
  \mathcal{E}_p(q)=c_p+\mathcal{E}^\star_p(q),
\end{equation}
with all terms finite by \cref{as:finite}. {Since $\tilde\Pi\ll\Pi$ and $\Pi(\Ssp\times\mathbb{N})=1$, $\tilde\Pi$-almost every pair $(p,n)$ has $p\in\Ssp$ and thus} has $p^\star=p$, $c_p=0$ and $\mathcal{E}^\star_p=\mathcal{E}_p$ (\cref{as:proj}). Hence $\Pcov_{\Pi'}(\Pi)=(1-\lambda)\,\E_{\Pi'_\perp}[c_p]$, and integrating \cref{eq:telescope} against \cref{eq:lebesgue} gives
\[
  R_{\Pi'}(q)-H_{\Pi'}=\lambda\,\E_{\tilde\Pi}[\mathcal{E}_p(q)]+\Pcov_{\Pi'}(\Pi)+(1-\lambda)\,\mathcal{O}_{\Pi'}(q).
\]
By \ref{f:com}, $\E_{\tilde\Pi}[\mathcal{E}_p(q)]=\E_\Pi[\mathcal{E}_p(q)]+\Gamma_{\Pi'}(q)$, and
\[
  \E_\Pi[\mathcal{E}_p(q)]=R_\Pi(q)-H_\Pi=\big[R_\Pi(q)-R_\Pi(\pi_\Pi)\big]+\big[R_\Pi(\pi_\Pi)-H_\Pi\big]=\mathcal{A}_\Pi(q)+\mathcal{C}_\Pi,
\]
which proves \cref{eq:decomp}. Finally, $\mathcal{A}_\Pi(q)\ge0$ by \cref{prop:bayes}, and \ref{f:kl} with $Z=(p,\bX,\Dn)$, $u=p(\cdot\mid\bX)$ and $w=\pi_\Pi$ gives $\mathcal{C}_\Pi=\E_\Pi[\KL(p(\cdot\mid\bX)\,\|\,\pi_\Pi(\cdot\mid\bX,\Dn))]\ge0$.
\end{proof}

\subsection{Proof of \cref{lem:ambiguity}}\label{app:proof-ambiguity}
\begin{proof}
For every $y$ with $p(y)>0$,
\[
\tfrac1K\textstyle\sum_k\big[\log p(y)-\log q_k(y)\big]-\big[\log p(y)-\log\bar q(y)\big]=\log\bar q(y)-\tfrac1K\sum_k\log q_k(y).
\]
Taking $\E_{y\sim p}$ gives \cref{eq:ambiguity}, all terms being finite by \cref{as:finite}. By strict concavity of $\log$, the right-hand side is non-negative, with equality if and only if $q_1(y)=\dots=q_K(y)$. Hence ${\Delta_p}\ge0$, with equality if and only if the $q_k$ agree $p$-almost everywhere.
\end{proof}

\subsection{Proof of \cref{cor:bias_variance}}\label{app:proof-corollary}
\begin{proof}
By \cref{prop:bayes} with $\Pi'=\Pi$, which holds for every predictor, $\mathcal{A}_\Pi(q)=\E_\Pi[\KL(\pi_\Pi\,\|\,q)]$ for each $q_k$ and for $\bar q$. Apply \cref{lem:ambiguity} at each $(\bX,\Dn)$ with $p=\pi_\Pi(\cdot\mid\bX,\Dn)$ and take $\E_\Pi$. All terms are non-negative, so the expectation splits.
\end{proof}

\subsection{Ensemble gain under the deployment distribution}\label{app:proof-deployment}
\begin{corollary}\label{cor:deployment_amb}
Let $q_1,\dots,q_K$ be PFNs with mixture $\bar q$ and, {for a pair $(p,n)$}, let $a_p:=\E_p\big[{\Delta_{p(\cdot\mid\bX)}}\big]$, where ${\Delta}$ from \cref{lem:ambiguity} is evaluated at the members' predictions $q_k(\cdot\mid\bX,\Dn)$. If all risks are finite and $\rt\,a_p$ is $\Pi$-integrable, then
\begin{align}
  \tfrac1K\textstyle\sum_{k}R_{\Pi'}(q_k)&-R_{\Pi'}(\bar q)=\E_{\Pi'}[a_p] \notag\\
  &=\lambda\Big[\underbrace{\E_\Pi\big[{\Delta_{\pi_\Pi}}\big]}_{\text{removable}}+\underbrace{\Cov_\Pi(\rt,a_p)}_{\text{transport}}\Big] \notag\\
  &\quad+\underbrace{(1-\lambda)\,\E_{\Pi'_\perp}[a_p]}_{\text{off-support}}, \label{eq:deployment_amb}
\end{align}
where the removable term is that of \cref{cor:bias_variance} and $\lambda$, $\rt$, $\Pi'_\perp$ are as in \cref{eq:lebesgue}.
\end{corollary}
\begin{proof}
Let $g:=\log\bar q(Y\mid\bX,\Dn)-\frac1K\sum_k\log q_k(Y\mid\bX,\Dn)$. Since all risks are finite, $\frac1K\sum_kR_{\Pi'}(q_k)-R_{\Pi'}(\bar q)=\E_{\Pi'}[g]$. By \ref{f:label}, $\E[g\mid p,\bX,\Dn]={\Delta_{p(\cdot\mid\bX)}}$, so $\E_{\Pi'}[g]=\E_{\Pi'}[a_p]$, which is the first equality. \Cref{eq:lebesgue} and \ref{f:com} give
\[
  \E_{\Pi'}[a_p]=\lambda\big(\E_\Pi[a_p]+\Cov_\Pi(\rt,a_p)\big)+(1-\lambda)\,\E_{\Pi'_\perp}[a_p].
\]
Finally, under $\Pi$, conditioning $g$ on $(\bX,\Dn)$ instead and using \ref{f:label} gives $\E[g\mid\bX,\Dn]={\Delta_{\pi_\Pi(\cdot\mid\bX,\Dn)}}$, so $\E_\Pi[a_p]=\E_\Pi[g]=\E_\Pi[{\Delta_{\pi_\Pi}}]$.
\end{proof}

\section{A tractable synthetic prior}\label{app:tracktablePrior}

The priors used in state-of-the-art PFNs are based on structural causal models, where random nodes and edges (together with functions they compute) are drawn according to some predefined scheme that together define a data-generating process with a certain underlying causal (latent) structure. To generate the feature and target observations, noise is then passed through this graph and the resulting feature and target nodes are appended to form a dataset from the prior \citep{tabicl1, tabicl2, tabPFN4_2025}. Let $\xi$ collect the SCM parameters, i.e.\ the graph, the node functions, the response node and the assignment of nodes to features. Each $\xi$ defines a task $p_\xi$, and the posterior $\Pi(\cdot\mid\Dn)$ is over $\xi$. MCMC methods evaluate the likelihood $p(\Dn\mid\xi)$ and the prior density of $\xi$ pointwise, up to a common constant \citep{brooks2011handbook}. An SCM prior defines $p_\xi$ only implicitly, as the pushforward of the exogenous noise through the structural equations. The following four properties of current SCM priors prevent the direct evaluation using MCMC methods:

\begin{itemize}[leftmargin=1.2em]
    \item \textbf{Latent nodes.} Only a subset of the nodes enters the dataset as features and label. With latent nodes $\mathbf{U}$ and rows conditionally independent given $\xi$, the likelihood that marginalizes over these latent nodes
    has no closed form for nonlinear node functions.

    \item \textbf{Degenerate and discontinuous likelihood.} Discretizing a node into a categorical feature or a class label turns its likelihood into an indicator or an interval probability and tree-based node functions \citep{tabicl1} are piecewise constant. The likelihood is then discontinuous in $\xi$, can vanish on most of the parameter space, and has zero or undefined gradients, which poses a severe problem for gradient-guided MCMC samplers that are needed in this high-dimensional setting \citep[e.g.,][]{RobnikMCLMC}.

    \item \textbf{Dataset-level operations.} TabPFN draws its class boundaries from the realized continuous targets \citep{tabPFN3_2025}. The rows are then not conditionally independent given $\xi$, and $p(\Dn\mid\xi)\neq\prod_{i=1}^n p_\xi(\bx_i,y_i)$. TabICLv2 rejects datasets that fail its predictability filters \citep{tabicl2}, which only adds to the fact that the posterior is intractable, as one would have to account for this rejection step in the likelihood.

    \item \textbf{Variable dimension and multimodality.} The parameter $\xi$ contains the DAG, the number of nodes (2 to 32 in TabICLv2), the response node, the assignment of nodes to features and the weights of the node functions. The number of DAGs grows super-exponentially in the number of nodes. Additionally, the dimension of $\xi$ changes with the graph.
\end{itemize}

In order to still get a high-quality reference PPD via MCMC sampling, we turn to the BNN prior introduced in \citet{hollmann2023tabpfn} that we construct in the following way: A ReLU network with two hidden layers of width 8 or 16 maps $d=4$ features to a logit $f_{\bm{w}}(\bx)$. All weights and biases $\bm{w}$ are iid $\mathcal{N}(0,1)$, the standard prior choice in BNNs \citep{JuliusInterplay}, and labels follow $y\mid\bx,\bm{w}\sim\mathrm{Bernoulli}\big(\sigma(f_{\bm{w}}(\bx))\big)$. The covariate law $p_{\bX}$ is one of two:

\begin{itemize}
    \item Gaussian features: $\bX\sim\mathcal{N}(0,I_d)$.
    \item SCM features: the TabICL SCM generator draws one MLP-based or tree-based SCM per table (probabilities 0.7 and 0.3). It converts columns to categorical codes with probability 0.2, clips at four standard deviations, and standardizes over the table. We discard the SCM's own label.
\end{itemize}

For each covariate law we pretrain TabICLv2 PFNs of three sizes (small, medium and large according to \cref{tab:architectures}) for $50{,}000$ steps of 64 tables. The number of rows per table is log-uniform or uniform on $[16, 4608]$, and a uniform fraction in $[0.1, 0.9]$ of the rows forms the context. This induces the size distribution $\Psi$. The number of features $d$ is uniform on $\{1,\dots,8\}$, zero-padded to 8. Every context size $n\le\lfloor 0.9 \cdot 4608 \rfloor = 4147$ therefore lies in $\supp\Psi$. With SCM features we pretrain only the smallest and largest PFN, both with uniform row counts.

We evaluate the PFNs on 512 test tasks per hidden width at context sizes $n \in \{4, 8, \dots, 4096\}$. The test tasks are drawn independently of pretraining.
\begin{itemize}
    \item Gaussian features, nested design. Each task has a pool of 4096 rows. The context $\Dn$ is the first $n$ rows of the pool, and all $n$ share the same 128 queries, so differences across $n$ reflect only the added rows.
    \item SCM features, independent design. The first $n$ rows of a processed table are not distributed like a pretraining table, because standardization and categorical conversion use the whole table. We therefore draw one independent table per task and $n$, with $\bm{w}$ shared across $n$. Each table has $n$ context rows and $Q_n = \min(9n, \max(128, \lceil n/9 \rceil))$ queries, so every evaluated table is an exact pretraining draw.
\end{itemize}

As a reference $\hat\pi_\Pi$ we approximate the PPD $\pi_\Pi(y\mid\bx,\Dn)$ of \cref{eq:bnn_ppd} with Microcanonical Langevin Ensembles \citep[MILE;][]{EmanuelMILE}. As MCMC samplers come with no finite-time convergence guarantees \citep{WellingSGLD}, and in the BNN literature it is additionally well known that a complete characterization of the whole posterior distribution is computationally intractable due to a combinatorially exploding number of posterior modes \citep{wiese2023towards}, we employ a very large sampling budget and run, per task, 24 chains from a deep-ensemble warm start \citep{EmanuelConnecting}, $50{,}000$ warm-up and $10{,}000$ sampling steps, where we keep every tenth draw. This way, we make sure that we capture a major part of the functional variability that the weight space posterior expresses. We then pool the chains with equal weight in a Bayesian model average. Correspondingly, in \cref{fig:bnn_saturation} it is clearly visible that adding further chains to the Bayesian model average saturates after roughly 12 chains.
Our metric is the prior-data NLL: the mean of $-\log q_{\bth}(y\mid\bx,\Dn)$ over a task's queries with their realized labels, averaged over tasks, i.e., an estimate of the risk \cref{eq:risk} at context size $n$. We also report the oracle predictive $p(y\mid\bx)=\sigma(f_{\bm{w}}(\bx))$, which has access to the true $\bm{w}$. Its risk at context size $n$ lower-bounds that of $\pi_\Pi$ by the context error at that size.

\begin{figure}
    \centering
    \includegraphics[width=\linewidth]{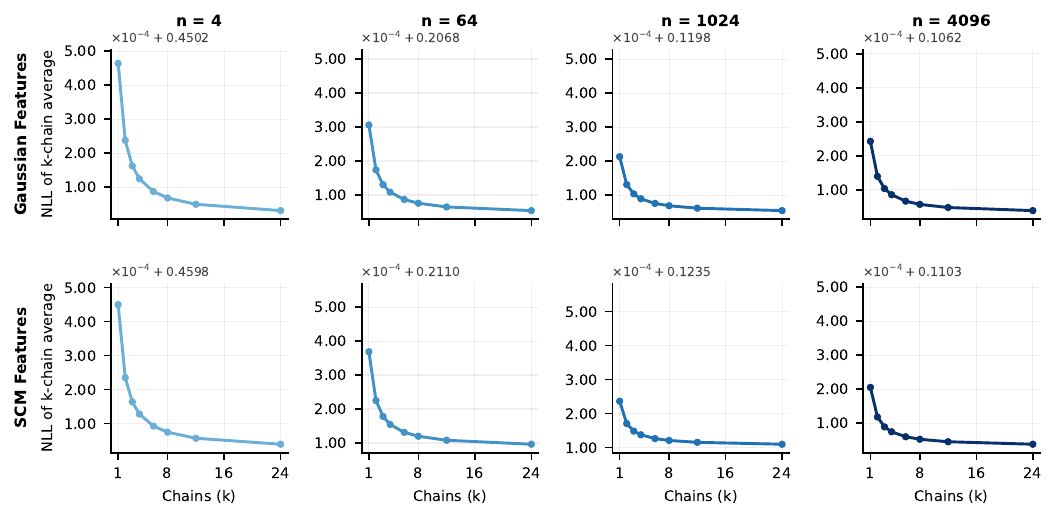}
    \caption{Holdout NLL of Bayesian model average of $k$ pooled MCMC chains for classification. The NLL saturates as more chains are added, which means a small number of chains seems sufficient to capture the functional diversity of the posterior distribution of the weights.}
    \label{fig:bnn_saturation}
\end{figure}

\paragraph{Where the amortization error concentrates}
To characterize the amortization error, we evaluate \cref{eq:amort}, $\KL(\hat\pi_\Pi\,\|\,q_{\bth})$, at each input $(\bx,\Dn)$. The inputs are 1024 additional rows per task for Gaussian features and the queries for SCM features (we cannot use additional rows for the SCM features due to the preprocessing issue explained above). We use the large PFN with hidden width 16 and uniform row counts (values for log-uniform row counts are given in parentheses).

The amortization error stems from few inputs, not from few tasks. The 10\% of inputs with the largest KL carry 60--66\% of the total at $n=64$ and over 80\% at $n\ge1024$, whereas the 10\% of tasks with the largest KL carry at most a third. These inputs lie near the decision boundary: at $n=4096$, the 20\% of inputs with the smallest predicted class-1 probability carry 80\% of the KL.

Relative to the between-chain variability of the reference, the amortization error grows with $n$. The PFN's KL to $\hat\pi_\Pi$ exceeds that of a single chain by a factor of 1.7 (3.0) at $n=64$ and 5.5 (11) at $n=4096$. At $n=4096$, the PFN predictive lies outside the range of all 24 chain predictives for over a third (two thirds) of inputs. The larger gap under log-uniform row counts is consistent with that law placing less pretraining mass on large contexts under $\Psi$. At large $n$, the PFN also shrinks confident logits, i.e., it becomes underconfident.

Aggregate miscalibration does not explain the amortization error: Against the oracle predictive $p(y\mid\bx)$, both $q_{\bth}$ and $\hat\pi_\Pi$ have calibration slopes within $0.03$ of 1 for $n\ge64$. A slope of 1 means that, averaged over inputs, the predictor is neither overconfident (slope < 1) nor underconfident (slope > 1). It is expected that the MCMC reference is well calibrated but it is interesting that the PFN is as well, as a predictor can be calibrated but still far from the PPD.

In summary, the amortization error concentrates at the decision boundary and is most severe at large context sizes that pretraining covers sparsely. The chains also disagree most at the decision boundary, so there the error of the PFN cannot be separated pointwise from the remaining error of the reference. In aggregate, however, the PFN's error exceeds the reference's approximation error due to finite-budget MCMC several-fold.

\begin{figure}
    \centering
    \includegraphics[width=\linewidth]{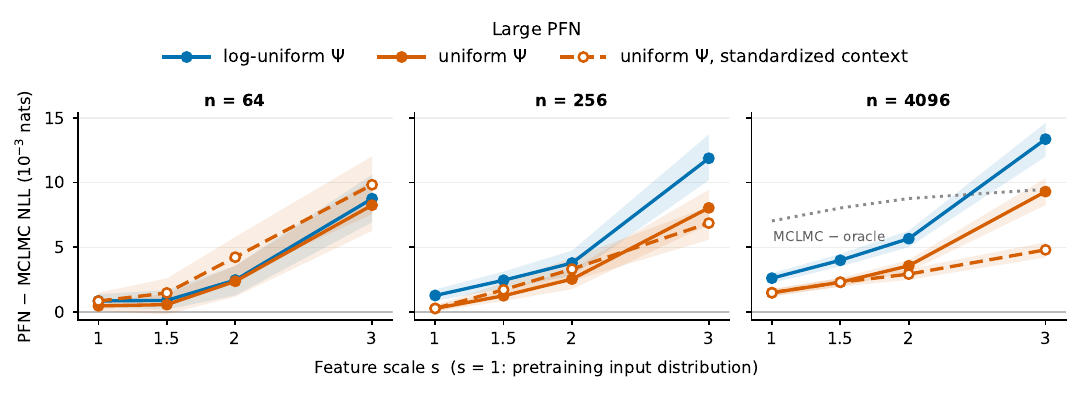}
    \caption{PFN to MCLMC gap under covariate shift. Test features are $\bx = s \bm{z}$ with $\bm{z} \sim \mathcal{N}(0, \bm{I}_4)$ and labels come from the BNN prior ($2\times8$ hidden neurons). The PFN was pretrained at $s = 1$. Each point is the mean difference over 256 test tasks. Bands are 95\% task-bootstrap intervals of each mean. Curves show the large PFN pretrained with uniform or log-uniform table size $\Psi$, and the uniform-$\Psi$ PFN with context-standardized features. Dotted grey ($n = 4096$ only): excess NLL of the MCLMC reference over the oracle.}
    \label{fig:pfn-mclmc-gap}
\end{figure}

\paragraph{Testing the PFN under deployment distribution shifts}

Pretraining minimizes the empirical risk only on average over the prior's input distribution. To test the PFN at inputs that are unlikely under the prior, we use the Gaussian feature prior and scale the features to $\bm{x} = s\bm{z}$ with $s \in \{1.5, 2, 3\}$ and draw labels from the prior's network. Because $\bm{w}$ is independent of $\bm{x}$, the posterior predictive under the training prior remains the target, and MCLMC estimates it with the same protocol. In \cref{fig:pfn-mclmc-gap}, the PFN-MCLMC gap in NLL grows with $s$ at every $n \geq 64$. At $n = 4096$, the difference is larger than the excess of MCLMC over the oracle. Under distribution shift, pretraining choices that do not matter much at $s = 1$ start to matter more, such as the more pronounced differences between the log-uniform $\Psi$ and the uniform $\Psi$. How closely the PFN matches the posterior predictive therefore depends on its pretraining input distribution.

\paragraph{Synthetic prior amortization gap for regression}

For regression, we draw tasks $y = f_{\bm{w}}(\bm{x}) + \sigma\varepsilon$ with $\log\sigma \sim \mathcal{N}(0, 1)$ and Gaussian features, train TabICL on raw labels, and compare it with MILE (MCLMC, 24 chains) on the same model for $n \leq 4096$. TabICL's regression head predicts 999 quantiles $q_1 \leq \dots \leq q_{999}$ at levels $\alpha_i = i/1000$, not a density. Its density is the derivative of the piecewise-linear quantile function: $q(y) = 10^{-3}/(q_{i+1} - q_i)$ for $y \in [q_i, q_{i+1})$, with exponential tails. The log density thus depends on the difference of two neighbouring outputs, and small errors in the predicted quantiles make it non-smooth. This inflates the NLL but barely moves the CRPS, which is linear in the quantile errors. Again, the PFN comes very close to the reference MCMC PPD across the whole range of pretraining context sizes.

The gap at intermediate context sizes comes almost entirely from tasks with low noise. In these tasks, the posterior concentrates quickly as data arrive, and the PFN lags behind: its predictive median lies further from the true function than MILE's, and it compensates with a wider predictive. Tasks with high noise show no gap at any context size, and the gap disappears once the data determine the function. The same low-noise tasks expose a weakness of the reference at the largest context sizes. There, MILE's equal-weight mixture over chains becomes over-dispersed, and the PFN predicts better than the reference. Part of the gap's decline at large $n$ therefore reflects the reference degrading, not only the PFN improving.

\begin{figure}[h]
    \centering
    \includegraphics[width=0.9\linewidth]{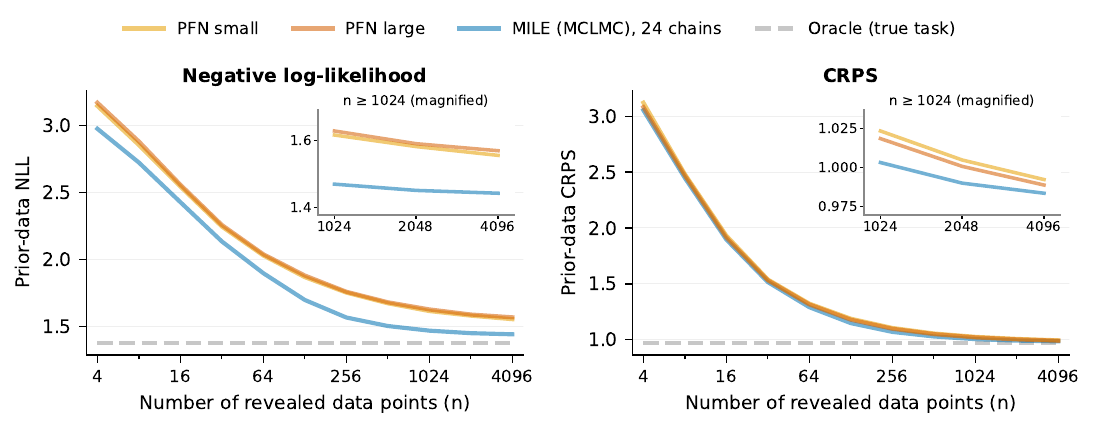}
    \caption{BNN regression with hidden width $8 \times 8$, $d = 4$, 512 test tasks: prior-data NLL (left) and Continuous Ranked Probability Score (CRPS; right) against the number of context points $n$ for the small and large PFN, the MILE (MCLMC) reference (24 chains), and the oracle. The PFN's NLL gap to MILE stems mostly from the density implied by its 999 predicted quantiles and on CRPS, which needs no density, the PFN nearly matches MILE.}
    \label{fig:pfn-mclmc-gap-reg}
\end{figure}

\section{Amortization disagreement metrics}\label{app:metrics}
We define the disagreement metrics across $K$ independently trained members $\{\bth^{(k)}\}_{k=1}^K$, evaluated on $N$ test points. All members share optimizer settings, $M$ and model size and only random seeds and pretraining data differ.

\paragraph{Jensen--Shannon divergence}
The generalized JSD \citep{JSD} splits the predictive entropy of the mixture $\bar q=\frac1K\sum_k q_{\bth^{(k)}}$,
\begin{equation}
    H(\bar{q}) = \mathrm{JSD}\big(q_{\bth^{(1)}}, \ldots, q_{\bth^{(K)}}\big) + \tfrac{1}{K}\textstyle\sum_k H\big(q_{\bth^{(k)}}\big),
    \qquad
    \mathrm{JSD}\big(q_{\bth^{(1)}}, \ldots, q_{\bth^{(K)}}\big)=\tfrac1K\textstyle\sum_k\KL\big(q_{\bth^{(k)}}\,\|\,\bar q\big).
\end{equation}
With equal weights, it equals the mutual information between the prediction and the member index. We report its average over test points,
\begin{equation}
    \widehat{AU}_{\mathrm{JSD}} = \tfrac{1}{N}\textstyle\sum_{i=1}^N \mathrm{JSD}\big(q_{\bth^{(1)}}(\cdot \mid \bx_i,\Dn), \ldots, q_{\bth^{(K)}}(\cdot \mid \bx_i,\Dn)\big).
\end{equation}
The ratio $\widehat{AU}_{\mathrm{JSD}} / H(\bar{q})$ expresses the disagreement as a fraction of the total predictive entropy.

\paragraph{Pairwise disagreement rate}
With $\hat y^{(k)}(\bx)=\argmax_y q_{\bth^{(k)}}(y\mid\bx,\Dn)$,
\begin{equation}
    \widehat{AU}_{\mathrm{dis}} = \tbinom{K}{2}^{-1}\textstyle\sum_{k<l} \tfrac{1}{N}\sum_{i=1}^N \mathbf{1}\big[\hat{y}^{(k)}(\bx_i) \neq \hat{y}^{(l)}(\bx_i)\big].
\end{equation}
Compared to the JSD, it registers disagreement only when the argmax flips, and hence underestimates functional disagreement; unlike the JSD, it measures disagreement at the level of decisions.

\paragraph{Estimation error}
$\widehat{EE}(M)$ is defined in \cref{eq:ee_metric}. It is a biased estimator of the estimation error, since the population minimizer $\bth^*$ is replaced by a reference checkpoint $\bth_{\mathrm{ref}}$ trained on $M_{\mathrm{ref}}\gg M_L$ datasets; a ground-truth minimizer of the population objective is not tractable.

\section{Experimental setup}\label{app:setup}
{\subsection{Implementation and data}}
\paragraph{Implementation}
The SG-MCMC experiments and the deep ensembles of \cref{sec:experiments} use our JAX \citep{jax} reimplementations of nanoTabPFN and TabICLv2 based on the official repository of \citet{EmanuelMILE}. IVON and the full-size TabICLv2 deep ensembles of \cref{sec:fullscale} are implemented in PyTorch \citep{paszke2019pytorch}, building on the official TabICL codebase \citep{tabicl2}. \Cref{tab:architectures} lists all architectures.

\paragraph{Evaluation data}
We hold out the 45 datasets of the TALENT-tiny core set \citep{liu2024talent} as validation data, which we use only to monitor training progress and not for checkpoint selection, and report all TALENT results on the remaining 255 datasets (170 classification, 85 regression), which we refer to simply as TALENT. {From TabArena \citep{TabArena}, the experiments of \cref{sec:experiments} use the 30 classification datasets within the scope of the small models (at most 10 classes and 100 features, a minority class of at least 1\% and at most 20\% missing values). Each dataset is evaluated with stratified 5-fold cross-validation, averaged over five random subsamples. The full-size models of \cref{sec:fullscale} are evaluated on all 38 classification and 13 regression datasets of TabArena. Finally, BeyondArena \citep{Purucker2026BeyondReally} collects datasets with grouped, clustered, and temporal structure that departs from the i.i.d.\ assumption of the synthetic prior; we use it to study member disagreement beyond the prior (\cref{app:beyond}).}

{\subsection{Bayesian model averages}\label{app:bma_setup}}
\paragraph{Deep ensembles (\cref{sec:experiments})}
All members are trained with AdamW \citep{loshchilov2019adamw} ($\beta_1{=}0.9$, $\beta_2{=}0.99$, weight decay $10^{-4}$) on batches of 64 synthetic datasets drawn from the TabICLv2 graph-SCM prior, each with 1024 rows, 1--100 features, up to 10 classes and a train fraction drawn from $[0.3, 0.9]$. The learning rate follows a cosine decay from its peak to $0.3\times$ the peak over $60\%$ of the maximum number of steps and stays constant afterwards; for nanoTabPFN, this decay is preceded by a linear warmup from $10\%$ of the peak over 600 steps. The optimizer settings were chosen via an extensive grid search on the small architecture size and an additional smaller grid search of the learning rate settings per model size. Training stops early after 10 epochs without improvement of the cross-entropy on 100 held-out synthetic validation batches, and each member keeps its weights with the lowest validation loss. Members are trained in parallel and never share a pretraining dataset; the size-dependent settings are given in \cref{tab:training_small}. The $M$-scaling runs (\cref{sec:results_scaling}) use the tiny and small recipes with $K{=}4$ members per value of $M$, where $M$ caps the number of unique datasets each member sees within the fixed step budget; the reference runs use $M{=}10^6$ (tiny) and $M{=}1.5\times10^6$ (small).

\begin{table}[h]
    \centering
    \caption{Size-dependent training settings of the deep ensembles in \cref{sec:experiments}. Maximum steps are reached only if early stopping does not trigger.}
    \label{tab:training_small}
    \begin{tabular}{lcccc}
        \toprule
        & tiny & small & medium & large \\
        \midrule
        Maximum epochs & 30 & 40 & 55 & 70 \\
        Maximum steps & 18{,}000 & 24{,}000 & 33{,}000 & 42{,}000 \\
        Peak learning rate, nanoTabPFN & $4\cdot10^{-4}$ & $4\cdot10^{-4}$ & --  & -- \\
        Peak learning rate, TabICLv2 & $4\cdot10^{-4}$ & $4\cdot10^{-4}$ & $4\cdot10^{-4}$ & $4\cdot10^{-4}$ \\
        Ensemble members $K$ & 8 & 8 & 8 & 4 \\
        \bottomrule
    \end{tabular}
\end{table}

\paragraph{{SG-MCMC}}
Both samplers target $p_\beta$ directly. Each chain is warm-started from one trained DE member \citep{EmanuelConnecting} and continues on pretraining data disjoint from that of all members. The samples of all chains are averaged with equal weight via \cref{eq:bde_pred}, forming a Bayesian deep ensemble \citep{EmanuelMILE}. The DE explores uncertainty between basins, whereas each chain gathers information within its warm-start basin. {\Cref{tab:sgmcmc_hps} lists the settings used for all reported runs. The hyperparameters for both samplers were chosen via a two-stage grid search. In a first coarser step, all listed parameters were searched. Afterwards, in a second step, the step sizes and temperatures were tuned per model size.}

\begin{table}[H]
  \centering
  \caption{{SG-MCMC settings. Sizes vary as tiny / small / medium / large. Both samplers use the same number of gradient evaluations per chain; pSMILE takes two per step.}}
  \label{tab:sgmcmc_hps}
  \small
  \begin{tabular}{lcc}
    \toprule
    & {adaSGHMC} & {pSMILE} \\
    \midrule
    {Step size} & {$2\cdot10^{-4}$} & {$10^{-4}$} \\
    {Warm-up steps} & {500} & {250} \\
    {Sampling steps} & {2500 / 3000 / 2500 / 3500} & {1250 / 1500 / 1250 / 1750} \\
    {Thinning} & \multicolumn{2}{c}{{every $10\%$ of the sampling steps (10 samples per chain)}} \\
    {Gradient evaluations per chain} & \multicolumn{2}{c}{{3000 / 3500 / 3000 / 4000}} \\
    {Chains} & \multicolumn{2}{c}{{8 / 8 / 8 / 4 }} \\
    {Temperature $T$} & \multicolumn{2}{c}{{0.01}} \\
    {Prior $p(\bth)$} & \multicolumn{2}{c}{{$\mathcal{N}(0,I)$}} \\
    {Batch size} & \multicolumn{2}{c}{{64 datasets}} \\
    \bottomrule
  \end{tabular}
\end{table}

\paragraph{{IVON}}
{IVON minimizes the Gibbs objective $\E_Q[\losse(\bth)]+\beta^{-1}\KL(Q\,\|\,p)$ over diagonal Gaussians $Q=\mathcal{N}(\bm{m},\mathrm{diag}(\bm{s}^2))$, and its effective-sample-size hyperparameter is exactly $\beta$ in \cref{eq:tempered}. Predictions average over $S$ draws from $Q$ and are compared with the posterior mean $\bm{m}$ ($S{=}0$ in \cref{tab:ivon_scaling}).}
{\Cref{tab:ivon_hps} lists the settings used for all reported runs. Similar to the hyperparameter search for the SG-MCMC samplers, the parameters for IVON were also selected in a two-step procedure: First, all listed parameters were coarsely searched, and in a second step the effective sample size and learning rate were jointly optimized per model size.}

\begin{table}[H]
  \centering
  \caption{{IVON settings. All sizes additionally share gradient clipping at $10$, a cosine schedule with $1\%$ linear warm-up decaying to $10^{-6}$, batches of 64 datasets and bfloat16 precision.}}
  \label{tab:ivon_hps}
  \small
  \begin{tabular}{lcccc}
    \toprule
    & {tiny} & {small} & {medium} & {large} \\
    \midrule
    {Effective sample size $\beta$} & {$1.92\cdot10^{6}$} & {$3.84\cdot10^{6}$} & {$6.4\cdot10^{6}$} & {$1.92\cdot10^{7}$} \\
    {Learning rate} & {1.0} & {1.0} & {0.3} & {0.3} \\
    {Training steps} & {300{,}000} & {200{,}000} & {200{,}000} & {200{,}000} \\
    {$\beta_1$} & \multicolumn{4}{c}{{0.9}} \\
    {$\beta_2$} & \multicolumn{4}{c}{{0.99999}} \\
    {Weight decay} & \multicolumn{4}{c}{{$10^{-4}$}} \\
    {Hessian initialization} & \multicolumn{4}{c}{{0.5}} \\
    {Clip radius} & \multicolumn{4}{c}{{$10^{-3}$}} \\
    \bottomrule
  \end{tabular}
\end{table}

\subsection{Full-size TabICLv2}
\paragraph{Full-size TabICLv2 deep ensembles (\cref{app:fullDE})}
We train $K{=}8$ members each for classification and regression, differing in the seeds of the weight initialization and of the synthetic data stream. Regression members use bias-free layer normalization and a quantile head with 999 quantiles. All members are trained with Muon \citep{jordan2024muon} applied to all parameters, weight decay $0.01$, batches of 64 datasets, bfloat16 precision and gradient clipping at $10$, on four NVIDIA H200 GPUs per member. Each stage starts from the previous stage's weights. Within a stage, the learning rate warms up linearly over the first $1\%$ of steps and then follows a cosine decay to $10^{-7}$. Following the curriculum of \citet{tabicl2}, training proceeds in three stages: (1) $130{,}000$ steps at learning rate $8\cdot10^{-4}$ on datasets with up to 1024 rows and a train fraction drawn from $[0.3, 0.9]$; (2) $10{,}400$ steps at $10^{-4}$ on 400--10{,}240 rows (log-uniform); and (3) $11{,}000$ steps at $2\cdot10^{-5}$ with gradient clipping at $1$ on 400--60{,}000 rows, both with a train fraction of about $0.8$. Datasets have 1--100 features and up to 10 classes throughout. Compared with the official recipe, stages 1 and 2 are shorter ($500{,}000$ and $40{,}000$ steps in \citet{tabicl2}), stage 3 is slightly longer ($10{,}000$ steps), and training uses bfloat16 instead of float32. The total number of training steps was chosen such that the total training over all three stages per ensemble member took about 48 hours across a node with NVIDIA H200 GPUs.

\paragraph{Attention modifications for bfloat16 training}
To train stably in bfloat16, we make two changes to the attention of TabICLv2. First, queries and keys are RMS-normalized per head before the dot product in the column, row and ICL transformers (QK-normalization; \citealp{henry2020qknorm}). Second, following \citet{pfeiffer2026tabh2o}, we bound the length-dependent base scale of QASSMax as $c_{\max}\tanh\big(g(\log n)/c_{\max}\big)$ with $c_{\max}{=}10$, where $g$ is the original base MLP and $n$ the sequence length, and compute it in float32. The unbounded MLP extrapolates linearly in $\log n$, so at sequence lengths beyond those seen in training it can saturate the attention and overflow; with the query modulation $1+\tanh(\cdot)\in[0,2]$, the total scale now lies in $[-2c_{\max}, 2c_{\max}]$. At initialization, $c_{\max}\tanh(x/c_{\max})\approx x$, so the bounded variant starts out equal to the original. Both changes add fewer than $2{,}000$ parameters.

\begin{table}[h]
    \centering
    \caption{Architecture hyperparameters. NanoTabPFN large is omitted due to its training cost; TabICLv2 full, used in \cref{app:fullDE}, follows the architecture of \citet{tabicl2} with added QK-normalization and bounded QASSMax scaling. In TabICLv2, the ICL stage operates on the concatenated CLS column tokens (ICL width $=$ embedding dimension $\times$ CLS column tokens), and every block's MLP hidden dimension is twice its width. All models predict up to 10 classes.}
    \label{tab:architectures}
    \begin{tabular}{lccccc}
        \toprule
        & tiny & small & medium & large & full \\
        \midrule
        \multicolumn{6}{l}{\textit{NanoTabPFN}} \\
        Parameters (rounded) & 1M & 2M & 4M & -- & -- \\
        Transformer layers & 4 & 4 & 5 & -- & -- \\
        Embedding dimension & 128 & 192 & 256 & -- & -- \\
        Attention heads & 4 & 6 & 8 & -- & -- \\
        MLP hidden dimension & 256 & 384 & 512 & -- & -- \\
        \midrule
        \multicolumn{6}{l}{\textit{TabICLv2}} \\
        Parameters (rounded) & 1.5M & 2M & 7.5M & 9M & 27.5M \\
        Blocks (column / row / ICL) & 2 / 2 / 4 & 3 / 3 / 5 & 3 / 3 / 5 & 3 / 3 / 6 & 3 / 3 / 12 \\
        Heads (column / row / ICL) & 4 / 4 / 6 & 4 / 4 / 6 & 8 / 8 / 8 & 8 / 8 / 8 & 8 / 8 / 8 \\
        Embedding dimension & 64 & 64 & 128 & 128 & 128 \\
        ICL width & 192 & 192 & 384 & 384 & 512 \\
        CLS column tokens & 3 & 3 & 3 & 3 & 4 \\
        Column inducing vectors & 64 & 64 & 128 & 128 & 128 \\
        Feature group size & 3 & 3 & 3 & 3 & 3 \\
        \bottomrule
    \end{tabular}
\end{table}

\subsection{Code}\label{sec:code}
We will make our code available in the near future.

\section{Further results}\label{app:furtherResults}

\subsection{Pretraining data scaling}\label{app:M-scaling}
We present the extended data-scaling analysis of \cref{sec:results_scaling} for the tiny models and the synthetic test set, showing how estimation error and member disagreement change with the number of pretraining datasets $M$.

\begin{figure}[H]
    \centering
    \includegraphics[width=0.95\linewidth]{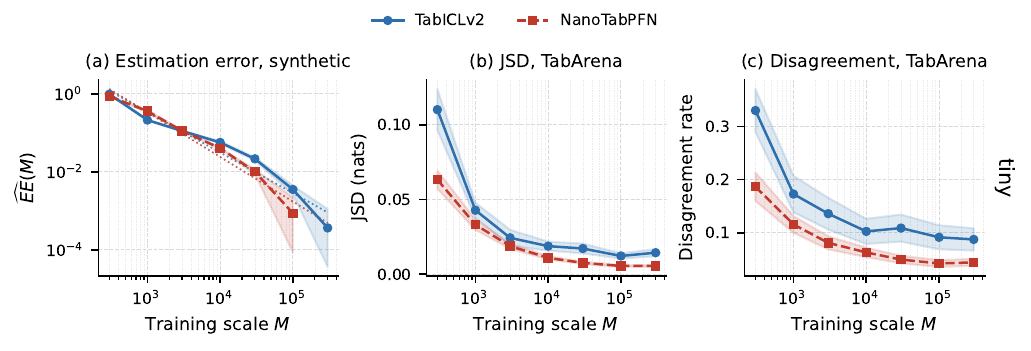}
    \caption{As \cref{fig:scaling_PFN_joint}, for the tiny models: estimation error $\widehat{EE}(M)$ (log-log) on synthetic data and member disagreement on TabArena versus the number of pretraining datasets $M$. Bands show $\pm1$ standard error over the $K{=}4$ members ($\widehat{EE}$) or over datasets (disagreement).}
    \label{fig:scaling_PFN_tiny}
\end{figure}

\begin{figure}[H]
    \centering
    \includegraphics[width=\linewidth]{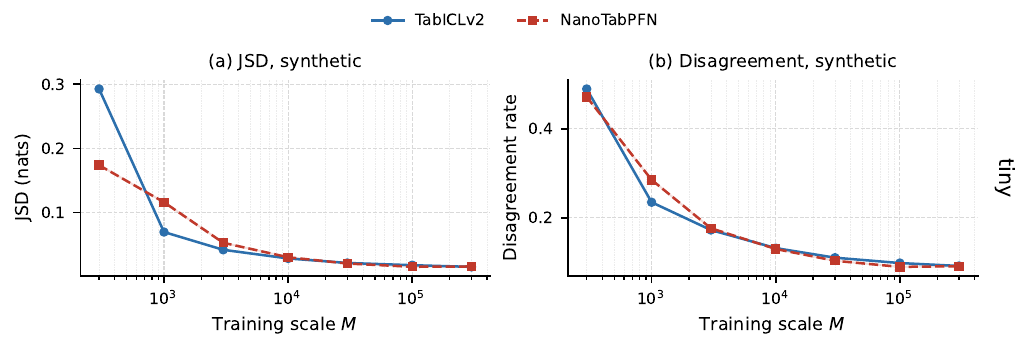}\\[2pt]
    \includegraphics[width=\linewidth]{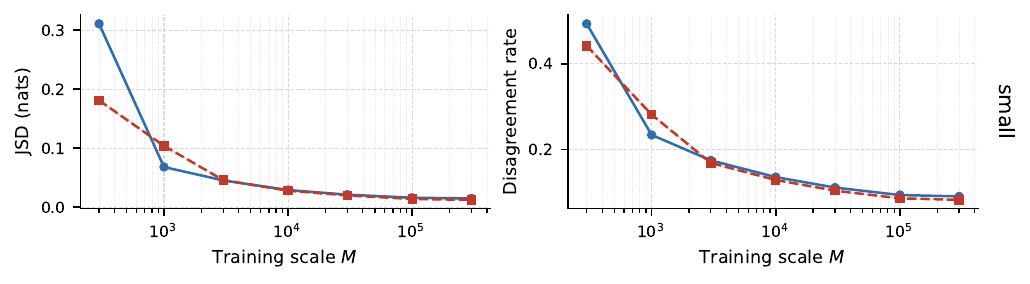}
    \caption{Member disagreement ($\widehat{AU}_{\mathrm{JSD}}$, $\widehat{AU}_{\mathrm{dis}}$) on the synthetic test set versus the number of pretraining datasets $M$, complementing the TabArena results of \cref{fig:scaling_PFN_joint}. Top: tiny, bottom: small models.}
    \label{fig:scaling_PFN_synthetic}
\end{figure}

\subsection{Model capacity}\label{app:capacity}
Experimental details on the capacity axis scaling testing whether larger models disagree less by comparing member disagreement across four TabICLv2 sizes on the synthetic test set drawn from the prior and on the TabArena classification datasets at two context sizes.

\begin{table}[H]
  \centering
  \caption{{Member disagreement ($K{=}4$) across TabICLv2 sizes on the synthetic test set and on the 30 TabArena classification datasets (mean $\pm$ standard error over datasets). $\widehat{AU}_{\mathrm{JSD}}$ in $10^{-3}$ nats, $\widehat{AU}_{\mathrm{dis}}$ in \%.}}
  \label{tab:capacity}
  \small
  \setlength{\tabcolsep}{4pt}
  \begin{tabular}{lcccccc}
    \toprule
    & \multicolumn{2}{c}{{synthetic}} & \multicolumn{2}{c}{TabArena, context 1000} & \multicolumn{2}{c}{TabArena, context 10000} \\
    \cmidrule(lr){2-3}\cmidrule(lr){4-5}\cmidrule(lr){6-7}
    Model & {$\widehat{AU}_{\mathrm{JSD}}$} & {$\widehat{AU}_{\mathrm{dis}}$} & $\widehat{AU}_{\mathrm{JSD}}$ & $\widehat{AU}_{\mathrm{dis}}$ & $\widehat{AU}_{\mathrm{JSD}}$ & $\widehat{AU}_{\mathrm{dis}}$ \\
    \midrule
    {tiny (1.5M)}   & {16.1} & {9.4} & $12.5 \pm 2.0$ & $7.5 \pm 1.4$ & $11.8 \pm 1.6$ & $8.0 \pm 1.5$ \\
    {small (2M)}    & {15.5} & {9.2} & $11.4 \pm 1.6$ & $7.5 \pm 1.5$ & $11.0 \pm 1.4$ & $7.8 \pm 1.5$ \\
    {medium (7.5M)} & -- & {--} & $9.7 \pm 1.5$ & $6.7 \pm 1.4$ & $10.5 \pm 1.3$ & $7.0 \pm 1.4$ \\
    {large (9M)}    & {13.6} & {8.2} & $7.4 \pm 1.4$ & $5.9 \pm 1.2$ & $8.5 \pm 1.4$ & $6.5 \pm 1.2$ \\
    \bottomrule
  \end{tabular}
\end{table}

\subsection{Disagreement concentration}\label{app:concentration}
To see where the averaging gain between independently pretrained TabICLv2 members comes from, we evaluate it per test point. For a test point $(\bx_i,y_i)$ with context $\Dn$, the pointwise gain
$$g_i=\log\bar q(y_i\mid\bx_i,\Dn)-\tfrac1K\textstyle\sum_{k=1}^K\log q_k(y_i\mid\bx_i,\Dn)\ \ge0 $$
is the log-probability that the mixture $\bar q$ assigns to the observed label minus the members' average log-probability of it. It is zero exactly when all members assign the same probability to $y_i$, and its mean over the test points is the log-loss improvement of $\bar q$ over its average member, i.e., the empirical counterpart of $\Delta_p$ in Lemma 1. We sort the test points by $g_i$ and plot the share of $\sum_i g_i$ carried by the fraction of points with the largest gain on the synthetic test set and on TabArena.

\begin{figure}[H]
    \centering
    \includegraphics[width=\linewidth]{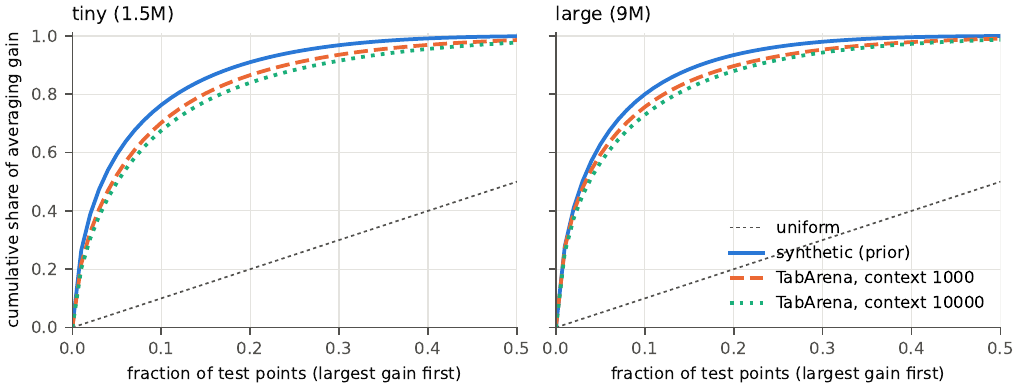}
    \caption{{Share of the total pointwise gain carried by the test points with the largest gain, for tiny (left) and large (right) TabICLv2 with K = 4 members. Dashed: gain spread evenly over all points.}}
    \label{fig:disagreement_concentration}
\end{figure}

\subsection{Deep ensembles}\label{app:DE}
We report the deep-ensemble scaling results for nanoTabPFN, complementing the TabICLv2 results of the main text.

\begin{figure}[H]
    \centering
    \includegraphics[width=\linewidth]{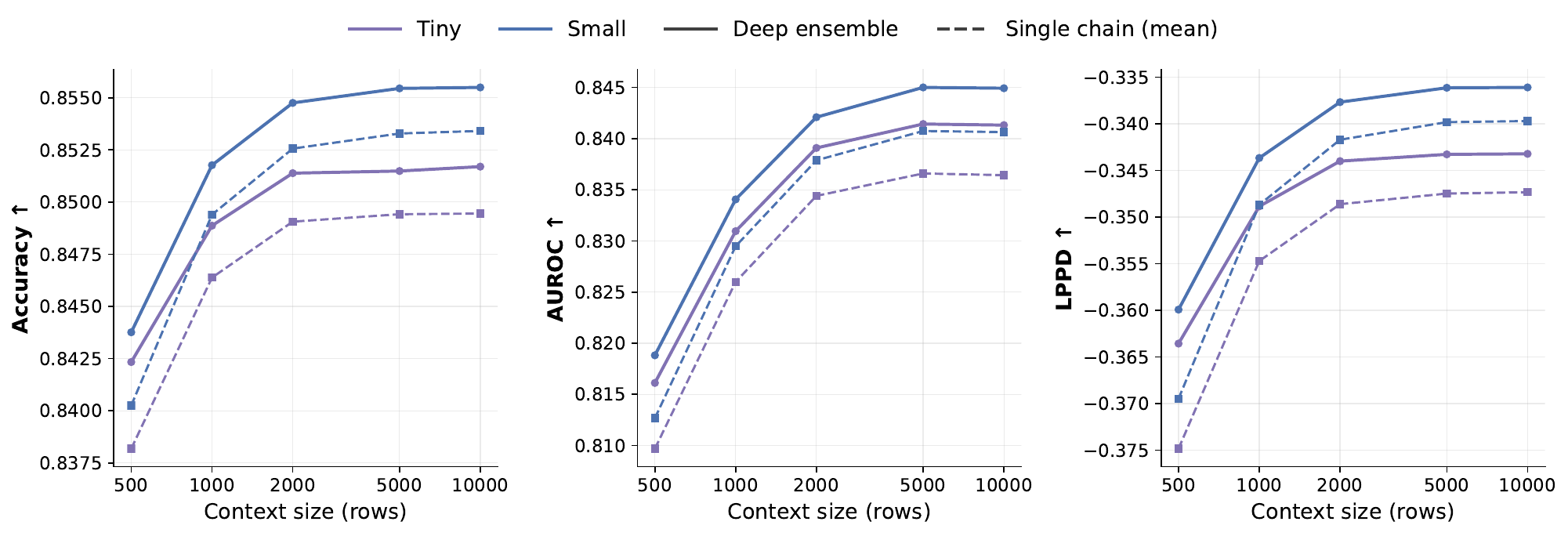}
    \caption{{As \cref{fig:DE-BDE-scaling}, for nanoTabPFN: accuracy, ROC AUC and LPPD on the 30 TabArena classification datasets across the maximal number of context rows, for deep ensembles and their average single member at the tiny and small model sizes.}}
    \label{fig:DE-scaling-nanoTabPFN}
\end{figure}

\subsection{Variational sampling via IVON}\label{app:IVON}
Additional evidence on TALENT classification to test whether sampling from IVON's variational posterior improves on its posterior mean, averaging $S$ posterior samples for three TabICLv2 sizes.

\begin{table}[H]
  \centering
  \caption{IVON variational sampling on the 170 TALENT classification datasets. $S=0$ is the posterior mean; $S\ge1$
  averages $S$ posterior samples without the mean. LPPD is the mean log predictive density per
  test point. $^{+}$/$^{-}$: significantly better/worse than the same model at $S=0$ (paired
  Wilcoxon over datasets, $p<0.05$).}
  \label{tab:ivon_scaling}
  \small
  \begin{tabular}{llccc}
    \toprule
    Model & $S$ & Accuracy $\uparrow$ & ROC AUC $\uparrow$ & LPPD $\uparrow$ \\
    \midrule
    tiny  & 0  & 0.8073 & 0.8737 & $-0.4464$ \\
          & 1  & 0.8035$^{-}$ & 0.8714$^{-}$ & $-0.4558^{-}$ \\
          & 8  & 0.8068 & 0.8738 & $-0.4493^{-}$ \\
          & 16 & 0.8063 & 0.8742 & $-0.4488^{-}$ \\
          & 32 & 0.8077 & 0.8752 & $-0.4484$ \\
          & 64 & 0.8072 & 0.8749 & $-0.4489$ \\
    \midrule
    small & 0  & 0.8129 & 0.8812 & $-0.4463$ \\
          & 1  & 0.8125 & 0.8811$^{-}$ & $-0.4488^{-}$ \\
          & 8  & 0.8122$^{-}$ & 0.8814 & $-0.4470^{-}$ \\
          & 16 & 0.8130 & 0.8816 & $-0.4465^{-}$ \\
          & 32 & 0.8134 & 0.8816 & $-0.4460^{+}$ \\
          & 64 & 0.8134 & 0.8817 & $-0.4459$ \\
    \midrule
    large & 0  & 0.8169 & 0.8824 & $-0.4338$ \\
          & 1  & 0.8153$^{-}$ & 0.8812$^{-}$ & $-0.4360^{-}$ \\
          & 8  & 0.8173 & 0.8824 & $-0.4332$ \\
          & 16 & 0.8171 & 0.8828$^{+}$ & $-0.4330$ \\
          & 32 & 0.8171 & 0.8827 & $-0.4330$ \\
          & 64 & 0.8173 & 0.8827 & $-0.4329$ \\
    \bottomrule
  \end{tabular}
\end{table}

\subsection{SG-MCMC sampling}\label{app:SGMCMC}
Extended results comparing the two SG-MCMC samplers, adaSGHMC and pSMILE, each warm-started from the deep-ensemble members, to the deep ensemble itself on the 30 TabArena classification datasets. Despite extensive hyperparameter searches, no parameters for adaSGHMC could be found that result in the sampler moving without at the same time decreasing performance compared to the DE.

\begin{table}[H]
  \centering
  \caption{{Relative LPPD gain (\%) of the Bayesian deep ensemble over the deep ensemble it is warm-started from, for both samplers, on the 30 TabArena classification datasets at context sizes 1000 and 10000. In parentheses: number of datasets on which the Bayesian deep ensemble improves on the deep ensemble. $^{\ast}$: significant change, paired Wilcoxon over datasets, $p<0.05$.}}
  \label{tab:sgmcmc_samplers}
  \small
  \begin{tabular}{lcccc}
    \toprule
    & \multicolumn{2}{c}{context size 1000} & \multicolumn{2}{c}{context size 10000} \\
    \cmidrule(lr){2-3}\cmidrule(lr){4-5}
    Model & adaSGHMC & pSMILE & adaSGHMC & pSMILE \\
    \midrule
    tiny   & $-0.05$ (11/30) & $+0.27^{\ast}$ (26/30) & $-0.05$ (11/30) & $+0.35^{\ast}$ (26/30) \\
    small  & $-0.64^{\ast}$ (4/30) & $+0.37^{\ast}$ (22/30) & $-1.12^{\ast}$ (5/30) & $+0.42^{\ast}$ (25/30) \\
    medium & -- & $+0.07$ (12/30) & -- & $+0.08$ (13/30) \\
    large  & $+0.12$ (18/30) & $+0.27^{\ast}$ (26/30) & $-0.52^{\ast}$ (9/30) & $+0.32^{\ast}$ (26/30) \\
    \bottomrule
  \end{tabular}
\end{table}

\subsection{Full-size TabICLv2 deep ensembles}\label{app:fullDE}
This section details the results summarized in \cref{sec:fullscale} where we ask whether the removable amortization error persists at the scale of current state-of-the-art PFNs, as well as with the full preprocessing pipeline. Since within-basin methods added little in \cref{sec:results_bde} and to reduce the computational costs, we restrict this section to deep ensembles. We also extend the analysis to regression, for which \Cref{prop:decomp,lem:ambiguity} hold unchanged with densities and differential entropies.

\textbf{Setup} We pretrain $K{=}8$ models that largely follow the architecture and training setup of TabICLv2 \citep{tabicl2} including the three-stage training curriculum, with details in \cref{app:setup}. Each member is run with the same inference-time ensemble of $8$ preprocessing and permutation configurations.

\textbf{Results} \textit{The gain persists at full scale.} After each of the three training stages, the ensemble improves on its average member in LPPD, Brier score and ROC AUC (\cref{fig:FullTabICL}a and \cref{tab:full_de_stages}). On the 170 TALENT classification datasets, the LPPD gain is $0.8\%$ at every stage. The ensemble is also significantly better than the best single member, selected after the fact, in all three metrics ($p\le2\cdot10^{-4}$ at stage~3). On TabArena, the ensemble improves on its average member on all 38 classification datasets and on the best member on 27 ($p=0.01$).

\textit{More training does not remove it.} All three stages improve every member, and most strongly on the large contexts that stages~2 and~3 train on: on contexts beyond 10{,}240 rows, the LPPD of the average member improves by $6.0\%$ from stage~1 to stage~3. Yet the ensemble gain stays at $0.7$--$0.9\%$ in every context-size bin and at every stage (\cref{tab:full_de_bins}). This is the full-scale counterpart of the data axis in \cref{sec:results_scaling}: further pretraining improves the members without reconciling them. The gain also does not come from the inference-time preprocessing. Giving each member 8 of 64 distinct preprocessing configurations, instead of sharing the same eight, leaves the ensemble unchanged (TabArena Elo $1555$ vs.\ $1552$). The gain thus stems from the independently pretrained weights, i.e., from amortization uncertainty, rather than from the context-level randomness that the preprocessing ensemble already averages over.

\textit{Ensembling as parallel pretraining.} A single member reaches a TabArena classification Elo of $1460$--$1514$. The eight-member ensemble reaches $1552$, within the confidence interval of the official TabICLv2 checkpoint ($1575$; \cref{fig:FullTabICL}c). Since the members never communicate, they can be trained concurrently on separate nodes. The ensemble then requires about $3\times$ the GPU hours of the official recipe at our throughput, but only about $0.37\times$ its wall-clock time, at $8\times$ the inference cost (\cref{tab:full_de_compute}). Ensembling (potentially shorter) independent pretraining runs is thus a simple way to parallelize PFN pretraining beyond a single node.

\begin{table}[H]
  \centering
  \caption{{Training budget of one full-size member against the official TabICLv2 recipe \citep{tabicl2}. Throughput was measured for our setup (bfloat16, four H200 GPUs, batch size 64); the official durations apply our throughput to its step counts and are therefore an estimate. GPU hours count four GPUs.}}
  \label{tab:full_de_compute}
  \small
  \begin{tabular}{lrrrrr}
    \toprule
    & \multicolumn{2}{c}{steps} & & \multicolumn{2}{c}{wall-clock (h)} \\
    \cmidrule(lr){2-3}\cmidrule(lr){5-6}
    Stage (max.\ rows) & official & ours & steps/h & official & ours \\
    \midrule
    1 (1024) & $500{,}000$ & $130{,}000$ & $6{,}033$ & $82.9$ & $21.5$ \\
    2 (10{,}240) & $40{,}000$ & $10{,}400$ & $2{,}150$ & $18.6$ & $4.8$ \\
    3 (60{,}000) & $10{,}000$ & $11{,}000$ & $650$ & $15.4$ & $16.9$ \\
    \midrule
    Total, one model & $550{,}000$ & $151{,}400$ & & $116.9$ & $43.3$ \\
    GPU hours, one model & & & & $467$ & $173$ \\
    GPU hours, 8-member ensemble & & & & & $1{,}386$ \\
    \bottomrule
  \end{tabular}
\end{table}

\begin{table}[H]
  \centering
  \caption{{Full-size TabICLv2 on TALENT after each training stage. \emph{Average member} averages each dataset's metric over the 8 members; \emph{best member} is the member with the best mean at that stage, selected after the fact. $^{\ast}$: deep ensemble significantly better than the best member (paired Wilcoxon over datasets, $p<0.05$).}}
  \label{tab:full_de_stages}
  \small
  \begin{tabular}{llccc}
    \toprule
    \multicolumn{5}{l}{\textit{Classification (170 datasets)}} \\
    Stage & & LPPD $\uparrow$ & Brier $\downarrow$ & ROC AUC $\uparrow$ \\
    \midrule
    1 & average member & $-0.3713$ & $0.2132$ & $0.9021$ \\
     & best member & $-0.3693$ & $0.2122$ & $0.9024$ \\
     & deep ensemble & $-0.3684^{\ast}$ & $0.2116^{\ast}$ & $0.9032^{\ast}$ \\
    \addlinespace
    2 & average member & $-0.3634$ & $0.2089$ & $0.9039$ \\
     & best member & $-0.3619$ & $0.2082$ & $0.9045$ \\
     & deep ensemble & $-0.3606^{\ast}$ & $0.2073^{\ast}$ & $0.9049^{\ast}$ \\
    \addlinespace
    3 & average member & $-0.3621$ & $0.2081$ & $0.9042$ \\
     & best member & $-0.3605$ & $0.2073$ & $0.9048$ \\
     & deep ensemble & $-0.3592^{\ast}$ & $0.2066^{\ast}$ & $0.9051^{\ast}$ \\
    \midrule
    \multicolumn{5}{l}{\textit{Regression (85 datasets)}} \\
    Stage & & nRMSE $\downarrow$ & RMSE $\downarrow$ & $R^2$ $\uparrow$ \\
    \midrule
    1 & average member & $0.4874$ & $18{,}394$ & $0.6728$ \\
     & best member & $0.4748$ & $17{,}495$ & $0.6887$ \\
     & deep ensemble & $0.4774$ & $18{,}211$ & $0.6846$ \\
    \addlinespace
    2 & average member & $0.4605$ & $17{,}202$ & $0.7030$ \\
     & best member & $0.4532$ & $16{,}925$ & $0.7085$ \\
     & deep ensemble & $0.4561$ & $17{,}073$ & $0.7070$ \\
    \addlinespace
    3 & average member & $0.4589$ & $17{,}218$ & $0.7043$ \\
     & best member & $0.4536$ & $17{,}061$ & $0.7087$ \\
     & deep ensemble & $0.4548$ & $17{,}106$ & $0.7081$ \\
    \bottomrule
  \end{tabular}
\end{table}

\begin{table}[H]
  \centering
  \caption{{Relative improvement (\%) of the deep ensemble over its average member per context-size bin (training rows available as context) on TALENT. Classification uses LPPD, regression nRMSE; positive values are improvements.}}
  \label{tab:full_de_bins}
  \small
  \begin{tabular}{lccc}
    \toprule
    & $\leq$1024 & 1024--10240 & $>$10240 \\
    \midrule
    \textit{Classification} & $n=33$ & $n=100$ & $n=37$ \\
    Stage 1 & $+0.85$ & $+0.70$ & $+0.86$ \\
    Stage 2 & $+0.85$ & $+0.77$ & $+0.72$ \\
    Stage 3 & $+0.85$ & $+0.78$ & $+0.75$ \\
    \midrule
    \textit{Regression} & $n=21$ & $n=43$ & $n=21$ \\
    Stage 1 & $+0.69$ & $+2.16$ & $+3.24$ \\
    Stage 2 & $+0.73$ & $+1.09$ & $+0.94$ \\
    Stage 3 & $+0.68$ & $+1.04$ & $+0.90$ \\
    \bottomrule
  \end{tabular}
\end{table}

\subsection{Ensemble size analysis}\label{app:ensembleSize}
To test how many members a full-size TabICLv2 deep ensemble needs, we evaluated every possible sub-ensemble of the eight members, from $K{=}1$ to $8$, and recorded the mean and standard deviation of LPPD (classification) and nRMSE (regression) across all sub-ensembles of each size on TALENT. Both metrics improve with every added member but quickly saturate (\cref{fig:abl-ensemble-size}): four members already achieve about 85\% of the full 8-member gain. Additional members mainly reduce the variance across sub-ensembles, so the result depends less on which members are included.
\begin{figure}[H]
    \centering
    \includegraphics[width=1\linewidth]{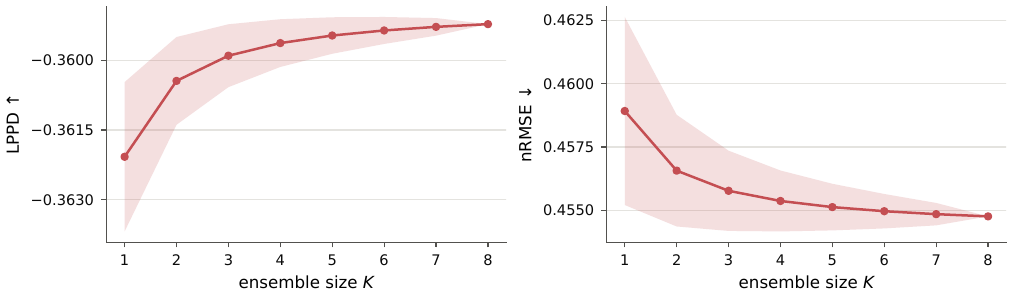}
    \caption{Full-size TabICLv2 deep ensembles on TALENT: LPPD on the 170 classification datasets (left) and nRMSE on the 85 regression datasets (right) versus ensemble size $K$. Lines show the mean, bands $\pm1$ standard deviation over all $\binom{8}{K}$ sub-ensembles of the eight stage-3 members.}
    \label{fig:abl-ensemble-size}
\end{figure}

\subsection{Disagreement beyond the prior}\label{app:beyond}
Results for testing the hypothesis that full-size TabICLv2 members disagree more outside the prior's support, comparing TabArena datasets within the support to BeyondArena datasets with shifted splits or beyond it.

\begin{figure}[H]
    \centering
    \includegraphics[width=0.75\linewidth]{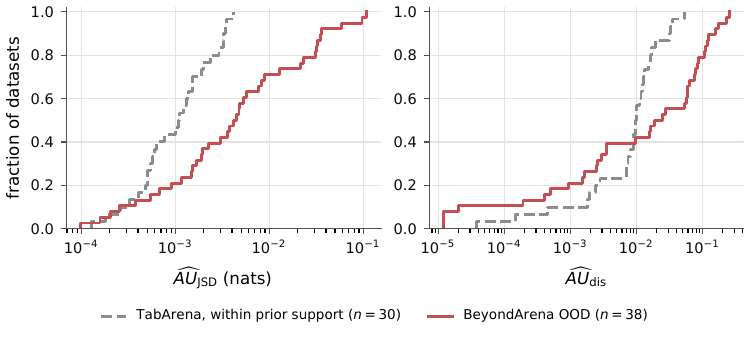}
    \caption{{Empirical CDFs over datasets of the disagreement between full-size TabICLv2 classification members, $\widehat{AU}_{\mathrm{JSD}}$ (left) and the pairwise disagreement rate $\widehat{AU}_{\mathrm{dis}}$ (right, log scale), on TabArena datasets within the prior's support (iid splits, at most 100 features, 10 classes and 48{,}000 training rows) and on BeyondArena datasets with a temporal or grouped split or outside that task space. Group summaries in \cref{tab:ood_groups}.}}
    \label{fig:ood_disagreement}
\end{figure}

\begin{table}[H]
  \centering
  \caption{{Member disagreement of full-size TabICLv2 classification members on TabArena datasets within the prior's support and on BeyondArena datasets beyond it. Prior support: at most 100 features, 10 classes and 48{,}000 training rows; BeyondArena OOD: a temporal or grouped split, or outside that task space. Text datasets are excluded. Each dataset's value is the mean over its splits; $p$: one-sided Mann--Whitney $U$ test, datasets are listed in \cref{tab:ood_tabarena_datasets,tab:ood_beyond_datasets}.}}
  \label{tab:ood_groups}
  \small
  \begin{tabular}{lll}
    \toprule
    & TabArena, within prior support & BeyondArena OOD \\
    \midrule
    Datasets & 30  & 38 \\
    Split & iid & 20 iid, 7 temporal, 11 grouped \\
    Features, median (range) & 20 (4--94) & 152 (3--22{,}215) \\
    Training rows, median (range) & 3{,}637 (499--47{,}679) & 30{,}625 (114--1{,}000{,}350) \\
    Classes, max & 5 & 20 \\
    Beyond prior support & 0 & 33 (23 $>$100 features, 17 $>$48{,}000 rows, 3 $>$10 classes) \\
    Median $\widehat{AU}_{\mathrm{JSD}}$ (nats) & $0.0011$ & $0.0043$ ($p=1.4\cdot10^{-4}$) \\
    Median $\widehat{AU}_{\mathrm{dis}}$ & $0.0099$ & $0.0218$ ($p=0.06$) \\
    \bottomrule
  \end{tabular}
\end{table}

\begin{table}[H]
  \centering
  \caption{{TabArena classification datasets within the prior's support (all iid splits): number of features, mean number of training rows per split and number of classes. The eight remaining TabArena classification datasets lie beyond the support (more than 100 features: APSFailure, Bioresponse, hiva\_agnostic, kddcup09\_appetency, MIC; more than 48{,}000 training rows: customer\_satisfaction\_in\_airline, GiveMeSomeCredit, SDSS17) and are not used for the experiments resulting in \cref{tab:ood_groups}.}}
  \label{tab:ood_tabarena_datasets}
  \small
  \begin{tabular}{lrrr}
    \toprule
    Dataset & Features & Train rows & Classes \\
    \midrule
    Amazon\_employee\_access & 9 & 21{,}846 & 2 \\
    anneal & 38 & 599 & 5 \\
    bank-marketing & 13 & 30{,}141 & 2 \\
    Bank\_Customer\_Churn & 10 & 6{,}667 & 2 \\
    blood-transfusion-service-center & 4 & 499 & 2 \\
    churn & 19 & 3{,}333 & 2 \\
    coil2000\_insurance\_policies & 85 & 6{,}548 & 2 \\
    credit-g & 20 & 667 & 2 \\
    credit\_card\_clients\_default & 23 & 20{,}000 & 2 \\
    diabetes & 8 & 512 & 2 \\
    Diabetes130US & 47 & 47{,}679 & 2 \\
    E-CommereShippingData & 10 & 7{,}333 & 2 \\
    Fitness\_Club & 6 & 1{,}000 & 2 \\
    hazelnut-spread-contaminant-detection & 30 & 1{,}600 & 2 \\
    heloc & 23 & 6{,}973 & 2 \\
    HR\_Analytics\_Job\_Change\_of\_Data\_Scientists & 12 & 12{,}772 & 2 \\
    in\_vehicle\_coupon\_recommendation & 24 & 8{,}456 & 2 \\
    Is-this-a-good-customer & 13 & 1{,}149 & 2 \\
    jm1 & 21 & 7{,}257 & 2 \\
    Marketing\_Campaign & 25 & 1{,}493 & 2 \\
    maternal\_health\_risk & 6 & 676 & 3 \\
    NATICUSdroid & 86 & 4{,}994 & 2 \\
    online\_shoppers\_intention & 17 & 8{,}220 & 2 \\
    polish\_companies\_bankruptcy & 64 & 3{,}940 & 2 \\
    qsar-biodeg & 41 & 703 & 2 \\
    seismic-bumps & 15 & 1{,}723 & 2 \\
    splice & 60 & 2{,}127 & 3 \\
    students\_dropout\_and\_academic\_success & 36 & 2{,}949 & 3 \\
    taiwanese\_bankruptcy\_prediction & 94 & 4{,}546 & 2 \\
    website\_phishing & 9 & 902 & 3 \\
    \bottomrule
  \end{tabular}
\end{table}

\begin{table}[H]
  \centering
  \caption{{BeyondArena classification datasets outside the prior's support or with a shifted (temporal or grouped) split, with the reason for inclusion. Training rows: mean per split. Datasets with more than 700{,}000 training rows are evaluated with a single preprocessing configuration instead of eight.}}
  \label{tab:ood_beyond_datasets}
  \scriptsize
  \begin{tabular}{llrrrl}
    \toprule
    Dataset & Split & Feat. & Train rows & Cls & Outside the prior \\
    \midrule
    amex\_non\_iid\_1m & grouped & 190 & 1{,}000{,}350 & 2 & shift, $>$100 feat., $>$48k rows \\
    asp\_potassco\_classification & grouped & 137 & 808 & 11 & shift, $>$100 feat., $>$10 cls \\
    cardiotocography & grouped & 23 & 1{,}417 & 3 & shift \\
    covertype & grouped & 14 & 341{,}750 & 3 & shift, $>$48k rows \\
    dementia\_prediction & grouped & 9 & 247 & 3 & shift \\
    mice\_protein\_trisomy\_discriminant & grouped & 77 & 720 & 8 & shift \\
    micro\_mass & grouped & 1{,}083 & 381 & 20 & shift, $>$100 feat., $>$10 cls \\
    musk & grouped & 167 & 4{,}399 & 2 & shift, $>$100 feat. \\
    pancreatic\_cancer\_mouse\_detection & grouped & 6{,}772 & 121 & 2 & shift, $>$100 feat. \\
    parkinsons\_biomedical\_voice\_measurements & grouped & 24 & 130 & 2 & shift \\
    sepsis\_prediction\_1m & grouped & 43 & 970{,}175 & 2 & shift, $>$48k rows \\
    \addlinespace
    acquire\_valued\_shoppers\_challenge & temporal & 111 & 109{,}766 & 2 & shift, $>$100 feat., $>$48k rows \\
    anes\_voting\_2026 & temporal & 318 & 27{,}917 & 2 & shift, $>$100 feat. \\
    ghanas\_indigenous\_intel & temporal & 10 & 7{,}442 & 4 & shift \\
    home\_credit\_default\_stability\_1m & temporal & 711 & 1{,}000{,}000 & 2 & shift, $>$100 feat., $>$48k rows \\
    hotel\_booking\_demand & temporal & 31 & 52{,}183 & 2 & shift, $>$48k rows \\
    ieee\_fraud\_detection & temporal & 435 & 321{,}131 & 2 & shift, $>$100 feat., $>$48k rows \\
    kick & temporal & 32 & 51{,}118 & 2 & shift, $>$48k rows \\
    \addlinespace
    aps\_failure & iid & 170 & 50{,}667 & 2 & $>$100 feat., $>$48k rows \\
    biogeographical\_ancestry\_prediction & iid & 104 & 423 & 10 & $>$100 feat. \\
    bioresponse & iid & 1{,}776 & 2{,}501 & 2 & $>$100 feat. \\
    clock\_protein\_toxicity & iid & 1{,}117 & 114 & 2 & $>$100 feat. \\
    customer\_satisfaction\_in\_airline & iid & 21 & 86{,}587 & 2 & $>$48k rows \\
    give\_me\_some\_credit & iid & 10 & 100{,}000 & 2 & $>$48k rows \\
    hiva\_agnostic & iid & 1{,}518 & 2{,}563 & 2 & $>$100 feat. \\
    home\_credit\_default\_risk & iid & 504 & 205{,}005 & 2 & $>$100 feat., $>$48k rows \\
    homesite\_quote\_conversion & iid & 295 & 173{,}835 & 2 & $>$100 feat., $>$48k rows \\
    kdd\_cup\_09\_appetency & iid & 212 & 33{,}333 & 2 & $>$100 feat. \\
    ljubljana\_primary\_tumor & iid & 17 & 201 & 11 & $>$10 cls \\
    lung\_cancer & iid & 12{,}600 & 131 & 4 & $>$100 feat. \\
    lung\_cancer\_epithelial\_genexp & iid & 22{,}215 & 125 & 2 & $>$100 feat. \\
    mic & iid & 111 & 1{,}133 & 8 & $>$100 feat. \\
    porto\_seguro & iid & 37 & 396{,}804 & 2 & $>$48k rows \\
    prostate\_cancer\_detection & iid & 15{,}154 & 215 & 2 & $>$100 feat. \\
    santander\_customer\_satisfaction & iid & 307 & 47{,}387 & 2 & $>$100 feat. \\
    santander\_customer\_transaction\_prediction & iid & 600 & 133{,}333 & 2 & $>$100 feat., $>$48k rows \\
    sdss\_17 & iid & 11 & 52{,}035 & 3 & $>$48k rows \\
    sepsis\_survival\_minimal\_clinical\_records & iid & 3 & 73{,}469 & 2 & $>$48k rows \\
    \bottomrule
  \end{tabular}
\end{table}

\end{document}